\documentclass[lettersize,journal]{IEEEtran}
\usepackage{amsmath,amsfonts}
\usepackage{amssymb, amsthm}
\usepackage{algorithmic}
\usepackage{algorithm}
\usepackage{array}
\usepackage{booktabs}
\usepackage{textcomp}
\usepackage{stfloats}
\usepackage{url}
\usepackage{verbatim}
\usepackage{graphicx}
\usepackage{cite}
\usepackage{siunitx}  
\usepackage{multirow}
\usepackage{subcaption}
\usepackage{xcolor}
\usepackage{tabularx}
\usepackage[
  font=normalsize,
  singlelinecheck=off
]{caption}

\newtheorem{assumption}{Assumption}
\newtheorem{proposition}{Proposition}
\newtheorem{theorem}{Theorem}

\newcolumntype{Y}{>{\centering\arraybackslash}X}

\newcolumntype{C}[1]{>{\centering\arraybackslash}p{#1}}

\usepackage{kotex}
\begin{document}

\title {Mitigating Resolution-Drift in Federated Learning: Case of Keypoint Detection }

\author{Taeheon Lim, Joohyung Lee~\IEEEmembership{Senior Member,~IEEE}, Kyungjae Lee~\IEEEmembership{Member,~IEEE}, Jungchan Cho~\IEEEmembership{Member,~IEEE}
\thanks{This work was supported by Institute of Information \& communications Technology Planning \& Evaluation (IITP)
A grant funded by the Korean government (MSIT) (No. 2022-0-00871, Development of AI Autonomy and Knowledge.
Enhancement for AI Agent Collaboration) and a National Research Foundation of Korea (NRF) grant funded by
The Korean Government (MSIT) (No. RS-2023-00211357).}

\thanks{Taeheon Lim is with the Department of Artificial Intelligence, Chung-Ang University, 84 Heukseok-ro, Dongjak-gu, Seoul, 06974, Republic of Korea (e-mail: icecoffee2500@cau.ac.kr).}
\thanks{Joohyung Lee and Jungchan Cho are with the Department of Computing, Gachon University, Seongnam 13120, Rep. of Korea (e-mail: j17.lee@gachon.ac.kr; thinkai@gachon.ac.kr).}
\thanks{Kyungjae Lee is with the Department of Statistics, Korea University, 145 Anam-ro, Seongbuk-gu, Seoul, 02841, Republic of Korea (e-mail: kyungjae\_lee@korea.ac.kr).}

\thanks{(Corresponding authors: Kyungjae Lee and Jungchan Cho)}}

\maketitle

\begin{abstract}
The Federated Learning (FL) approach enables effective learning across distributed systems, while preserving user data privacy. To date, research has primarily focused on addressing statistical heterogeneity and communication efficiency, through which FL has achieved success in classification tasks. However, its application to non-classification tasks, such as human pose estimation, remains underexplored.
This paper identifies and investigates a critical issue termed ``resolution-drift,'' where performance degrades significantly due to resolution variability across clients. Unlike class-level heterogeneity, resolution drift highlights the importance of resolution as another axis of not independent or identically distributed (non-IID) data. To address this issue, we present resolution-adaptive federated learning (RAF), a method that leverages heatmap-based knowledge distillation. Through multi-resolution knowledge distillation between higher-resolution outputs (teachers) and lower-resolution outputs (students), our approach enhances resolution robustness without overfitting.
Extensive experiments and theoretical analysis demonstrate that RAF not only effectively mitigates resolution drift and achieves significant performance improvements, but also can be integrated seamlessly into existing FL frameworks. 
Furthermore, although this paper focuses on human pose estimation, our t-SNE analysis reveals distinct characteristics between classification and high-resolution representation tasks, supporting the generalizability of RAF to other tasks that rely on preserving spatial detail.
\end{abstract}


\section{Introduction}
\label{sec:intro}

\IEEEPARstart{T}{he} rapid increase in edge devices' computational power has enabled local training on Internet of Things (IoT) devices using locally collected data~\cite{fedadg, ce-fedavg}. This paradigm shift facilitates machine learning without transmitting raw data to a central server, thereby overcoming the limitations of traditional centralized learning approaches, which often rely on constrained public datasets.
Consequently, distributed machine learning~\cite{edge_intelligence, edge_computing, decentralized_edge} has become an essential technology in edge computing and Internet of Things (IoT) environments. Federated Learning (FL)~\cite{fedavg, federated_learning} embodies this paradigm while preserving user data privacy.
FedAvg~\cite{fedavg} enables multiple clients' models to undergo decentralized training while preserving their private data and aggregating local updates into a shared global model.

\begin{figure}[t]
    \centering    
    \includegraphics[width=1.0\linewidth]{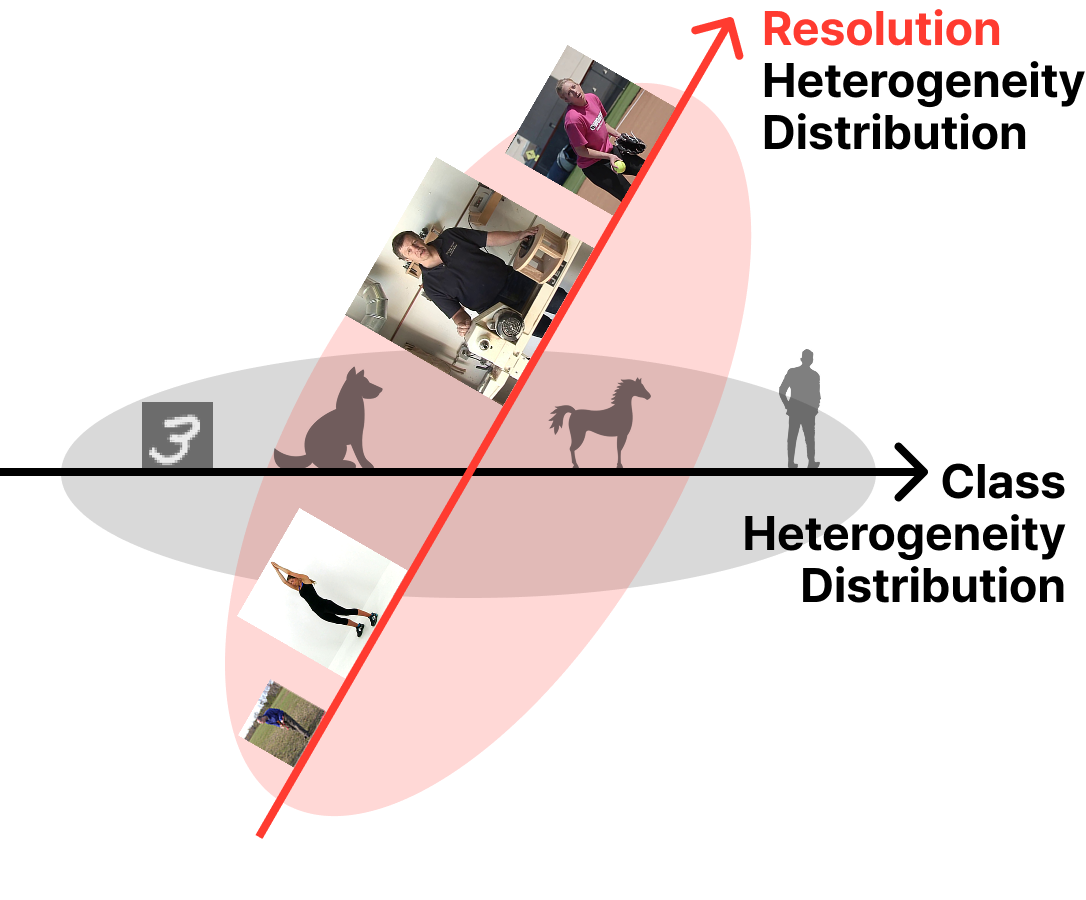}    
    \caption{Multiple axes defining statistical heterogeneity: the class axis represents the dataset's class distribution, and the resolution axis represents the distribution based on image resolutions.}
    \label{distribution-axis}
\end{figure}

Federated learning often suffers substantial performance degradation when client data distributions are heterogeneous.
FedProx~\cite{fedprox} and SCAFFOLD~\cite{scaffold} are representative methods for handling the statistical heterogeneity in FL.
However, these advances address non-IID data primarily in classification settings. 
Real-world applications often require capabilities beyond classification, such as object detection, human pose estimation, and depth estimation, client datasets vary regarding class labels and image resolution (see Figure~\ref{distribution-axis}).
Therefore, a pressing need arises to explore the applicability of FL to these non-classification tasks, an area that remains significantly under-investigated.
Specifically, high-resolution regression problems such as landmark or keypoint detection (e.g., Human Pose Estimation~\cite{simple_baseline, vitpose, deep-hrnet}) require fundamentally different network architectures, often adopting encoder-decoder structures. 
The decoder must recover the encoded features back to the input resolution, requiring the preservation of spatial information throughout the network. By contrast, classification models use an encoder-only architecture and discard most spatial details because the input image's spatial information need not be reflected in a single-class label.

In such high-resolution regression scenarios, all the participating clients are unlikely to possess data with the same resolution.
For example, the server aggregates the locally trained model weights derived from both low- and high-resolution images.
Multi-resolution aggregation is a fundamentally different issue to that which has been extensively studied in existing FL research on statistical heterogeneity~\cite{fedprox, scaffold}, where heterogeneity is often implicitly confined to class distribution skew.
A critical challenge for high-resolution regression in such settings is the feature representation mismatch caused by information loss or distortion owing to resolution differences, rendering these tasks highly sensitive to variations in data resolution.
As discussed in Section~\ref{sec:pre-analysis}, high-resolution regression tasks demand much deeper spatial features than classification. Therefore, heterogeneity along axes other than the class can cause severe performance drift.
Previous studies have noted multiple axes of non-IIDness~\cite{redifne_non-IID}; however, resolution heterogeneity remains under-investigated. This new distribution shift axis requires dedicated methods for mitigating its effects. Although there are other forms of distribution heterogeneity, this paper focuses specifically on resolution heterogeneity as a critical factor in high-resolution regression tasks.

To our knowledge, this multi-resolution issue has not been explicitly addressed in FL literature. We introduce the term \textbf{``resolution-drift''} to describe this performance degradation phenomenon. Resolution drift occurs when clients train on data with different resolutions, causing the global model to overfit to certain resolutions and lose its ability to generalize across others.
Existing FL methods designed to address data heterogeneity~\cite{fedprox, scaffold} fail to address the core issue, since they focus solely on the global aggregation step and overlook local resolution-induced overfitting.
For FL to be widely adopted in practical settings, its effectiveness must extend beyond class-level predictions to supporting pixel- or coordinate-level tasks, which are central to high-resolution regression problems such as analyzing CCTV footage collected from cameras with varying resolutions. Therefore, new FL methodologies must be developed, which can maintain a stable performance across diverse resolution inputs. Addressing this issue requires novel approaches beyond the conventional FL techniques.

This paper proposes \textbf{Resolution Adaptive Federated Learning (RAF)} to mitigate the resolution-drift problem for non-classification tasks. Our system architecture assumes that clients possess data with different resolutions (Figure~\ref{fig:fl-overview}).
To counteract resolution drift, RAF employs a heatmap-based Knowledge Distillation (KD) strategy, where a KD loss function minimizes the distance between outputs generated from higher-resolution inputs (teacher) and those from lower-resolution inputs (student).
By serving as soft targets, the teacher's output acts as a regularizer, preventing overfitting to any single resolution, thereby, enhancing robustness across multiple resolutions. This is consistent with the findings in the literature~\cite{hinton_kd, self_distillation, kd_lsr}.
However, designing a regularizer based on knowledge distillation across multiple resolutions using a transformer backbone introduces additional complications. The self-attention mechanism in transformers is inherently permutation-invariant; therefore, positional embeddings are required to inject spatial order information. Vision Transformer (ViT), our selected backbone relies on Absolute Positional Embedding (APE), whose shape is fixed by the input resolution and cannot adapt during training. Consequently, training ViT on multi-resolution inputs is complicated.
To address this limitation, we drew inspiration from recent work, ResFormer~\cite{resformer}, which replaces APE with convolution-based positional embeddings. In ResFormer, the convolutional kernels are learned, allowing positional embedding to adapt dynamically to different input resolutions and inject a smooth spatial context. 
Since ResFormer demonstrated this convolution-based positional embedding approach only in the context of classification tasks, its effectiveness for high-resolution regression remains unclear. 
Our experiments confirm that these embeddings are also effective for non-classification tasks, enabling robust feature extraction across varying resolutions.

\subsection*{Main Contributions}

Our main contributions are summarized as follows:
\begin{itemize}
    \item We identify the ``resolution-drift'' phenomenon in multi-resolution federated learning and formally define its impact on non-classification tasks such as high-resolution regression.
    \item We extend convolution-based positional embeddings to high-resolution regression, showing that they enable Vision Transformer backbones to train effectively on multi-resolution inputs.
    \item We present RAF, a novel framework that integrates multi-resolution knowledge distillation as a resolution-regularizer within standard FL, and provide a theoretical analysis of its convergence.
    \item Through extensive experiments, we demonstrate that RAF mitigates resolution-drift and allows single-resolution clients to benefit from a globally trained model that remains robust across all resolutions.
    \item RAF is modular and orthogonal to existing FL aggregation schemes (e.g., FedProx), facilitating easy integration into a wide range of FL pipelines.
\end{itemize}

\section{Related Works}

\subsection{Federated Learning (FL)}
Federated Learning is a distributed-learning paradigm that enables training across multiple clients without transferring their data to a central server, thereby preserving data privacy~\cite{fedavg}. FL has become an essential technology in edge computing and IoT environments because it facilitates collaborative learning while addressing privacy concerns.
Existing FL research has primarily focused on solving the challenges related to statistical heterogeneity and communication efficiency. FedProx~\cite{fedprox} has extended FedAvg by adding a proximal term to each client's objective, thereby enabling more stable convergence under non-IID data distributions in classification settings. SCAFFOLD~\cite{scaffold} further enhances this by introducing control variates to correct client drift during local updates, significantly reducing communication rounds, and improving the accuracy in distributed classification tasks. 
However, these studies~\cite{fedprox, scaffold} mainly concentrated on class-level heterogeneity and did not address resolution-level heterogeneity.
Unlike classification tasks, which predict class-level labels, high-resolution regression tasks require pixel-level label predictions, making them inherently more challenging.

In real-world applications, FL systems encounter devices with significantly different sensing capabilities. For example, in autonomous driving fleets, some vehicles are equipped with high-definition cameras, whereas others have low-resolution dashcams or infrared sensors. In precision agriculture, data may be collected using both high- and low-altitude drones, all of which produce images with different resolutions. Existing FL approaches~\cite{fedprox, scaffold, fednova, moon, feddyn} primarily address class-level heterogeneity but do not consider resolution-level heterogeneity. This limitation underscores the need for FL methodologies that can achieve robust learning in multi-resolution environments.

\subsection{Vision Transformer in Federating Learning}
\label{subsec:vit}

Transformer architectures, originally introduced by Vaswani et al.~\cite{transformer} and now ubiquitous across NLP~\cite{transformer, bert, gpt4, llama, lcm}, vision~\cite{vit, deit, dit, dino_v2}, audio~\cite{wav2vec, audiolm, ast}, and multimodal~\cite{clip, blip, llava} domains, offer key advantages for FL: their self-attention mechanism captures long-range dependencies, their modular design adapts easily to diverse tasks, and they scale gracefully with data size. Recent FL research has begun to leverage these strengths, achieving improved robustness and convergence with non-IID data~\cite{fl_transformer, fedyolo}.

We build on this trend by adopting a Vision Transformer (ViT)~\cite{vit} as the backbone of the proposed method. ViT~\cite{vit} processes images by splitting them into flattened patches. However, ViT relies on APE with its size fixed to a single training resolution, which hampers generalization to other resolutions and often requires interpolation at the inference stage~\cite{resformer}. To overcome this limitation, we replace APE with convolution-based positional embeddings from ResFormer~\cite{resformer}, which dynamically injects the spatial context across varying input sizes without explicit resizing. Although prior work has demonstrated this strategy primarily for classification, we extend it to a high-resolution regression task within a federated learning setting, demonstrating that it significantly improves human pose estimation (HPE) performance and resolution robustness.

\section{In-depth Analysis of Resolution Effects in Federated Learning}
\label{sec:pre-analysis}

Although extensive FL research has addressed the challenge of statistical heterogeneity in non-IID settings, most efforts remain limited to two key aspects.
\begin{itemize}
    \item Previous works have predominantly addressed classification tasks.
    \item Statistical heterogeneity is typically considered solely in terms of class distribution skew.
\end{itemize}
In this section, we discuss these limitations and underscore the need to address resolution heterogeneity in real-world FL scenarios.

\begin{figure}[t]
  \centering

  \begin{subfigure}[b]{\linewidth}
    \centering
    \includegraphics[width=\linewidth]{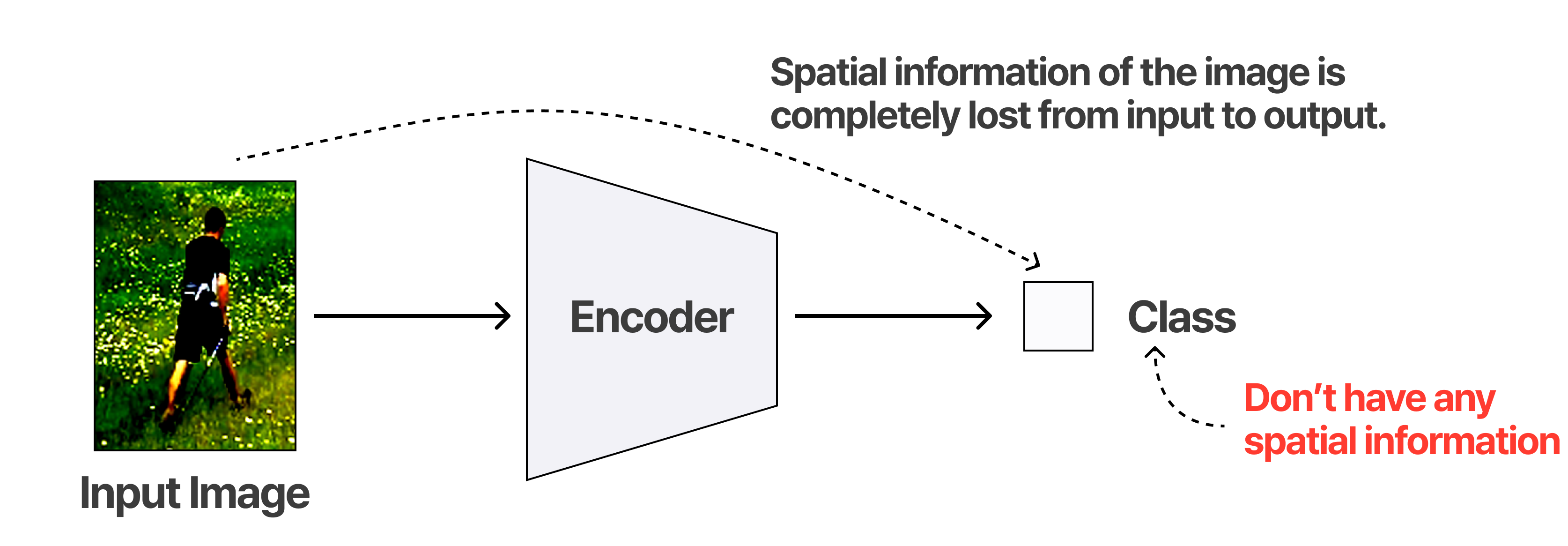}
    \subcaption{Classification}
    \label{fig:classification}
  \end{subfigure}

  \medskip

  \begin{subfigure}[b]{\linewidth}
    \centering
    \includegraphics[width=\linewidth]{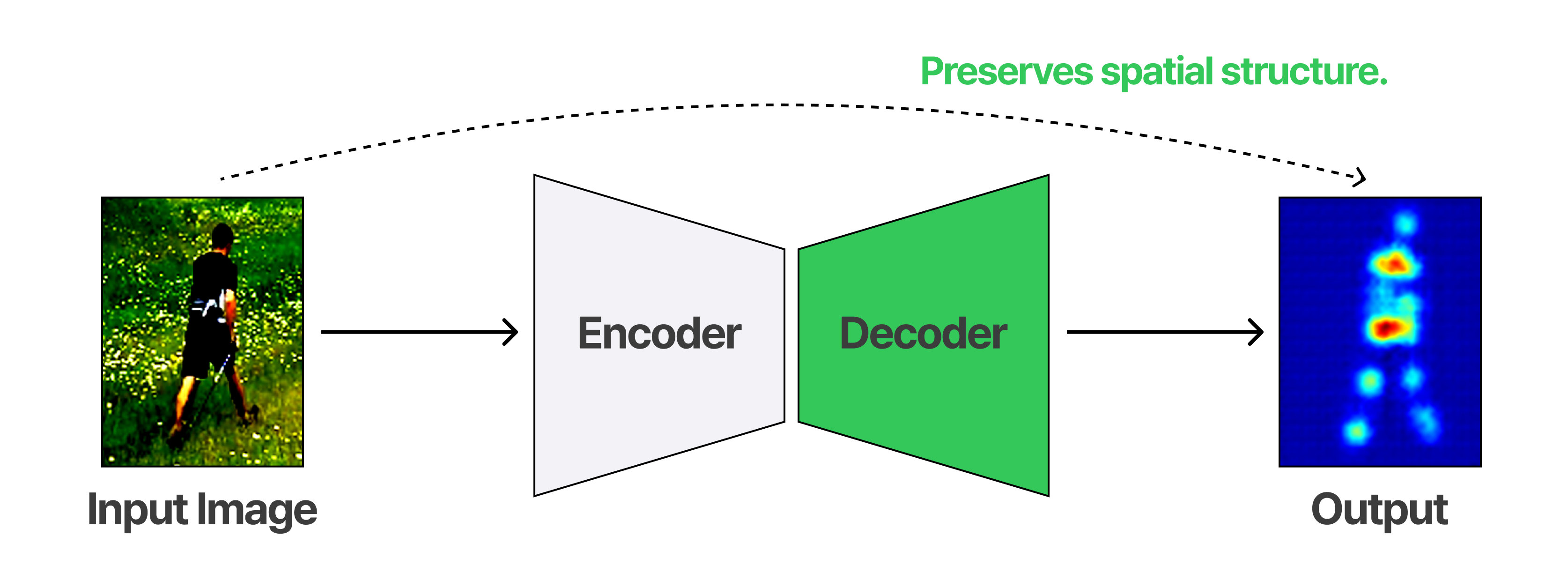}
    \subcaption{High-Resolution Regression}
    \label{fig:high-resolution regression}
  \end{subfigure}

  \caption{The architectural difference between classification and high-resolution regression}
  \label{fig:cls-hrr}
  \vspace{-15pt}
\end{figure}

\subsection{Fundamental Differences Between Classification and High-Resolution Regression}
\label{subsec:cls-hrr}

In real-world FL scenarios, classification tasks represent only a small fraction of use cases. The most fundamental vision problems involve high-resolution regression, including keypoint detection, depth estimation, and super-resolution. These tasks underpin many more real-world applications; however, they pose far greater challenges than classification and have received relatively little attention from FL researchers. This gap motivated us to focus on high-resolution regression in federated settings.

Figure~\ref{fig:cls-hrr} illustrates the architectural differences between the classification and high-resolution regression models. 
As shown in Figure~\ref{fig:classification}, classification corresponds to image-level prediction in which a single class label is predicted for each input image. In this setting, the predicted output does not contain any spatial information even though the input is a rich spatial map. Consequently, the network progressively shrinks the spatial dimensions of its feature maps and distills all relevant information into a single vector.
By contrast, high-resolution regression requires predictions that align with spatial locations in the original image, as illustrated in Figure~\ref{fig:high-resolution regression}. These high-resolution regression problems must preserve spatial detail throughout the network and produce output feature maps at or near the input resolution. Consequently, architectures for these tasks often adopt an encoder-decoder design, extract deep features, and then reconstruct or upsample them to full image size.

Because the output itself must retain a spatial structure, models for these tasks are inherently more sensitive to changes in the input resolution and demand richer feature representations than models trained solely for classification. In keypoint detection models, for instance, encoder-decoder architectures are designed with skip connections specifically to preserve high-resolution information; such spatial fidelity is critical for accurate boundary localization~\cite{hrnet-seg}.

\subsection{Experimental Analysis of Resolution Mismatch Effects}
\label{subsec:risk_of_res-hetero}

\begin{figure}[t]
    \centering
    \includegraphics[width=0.95\linewidth]{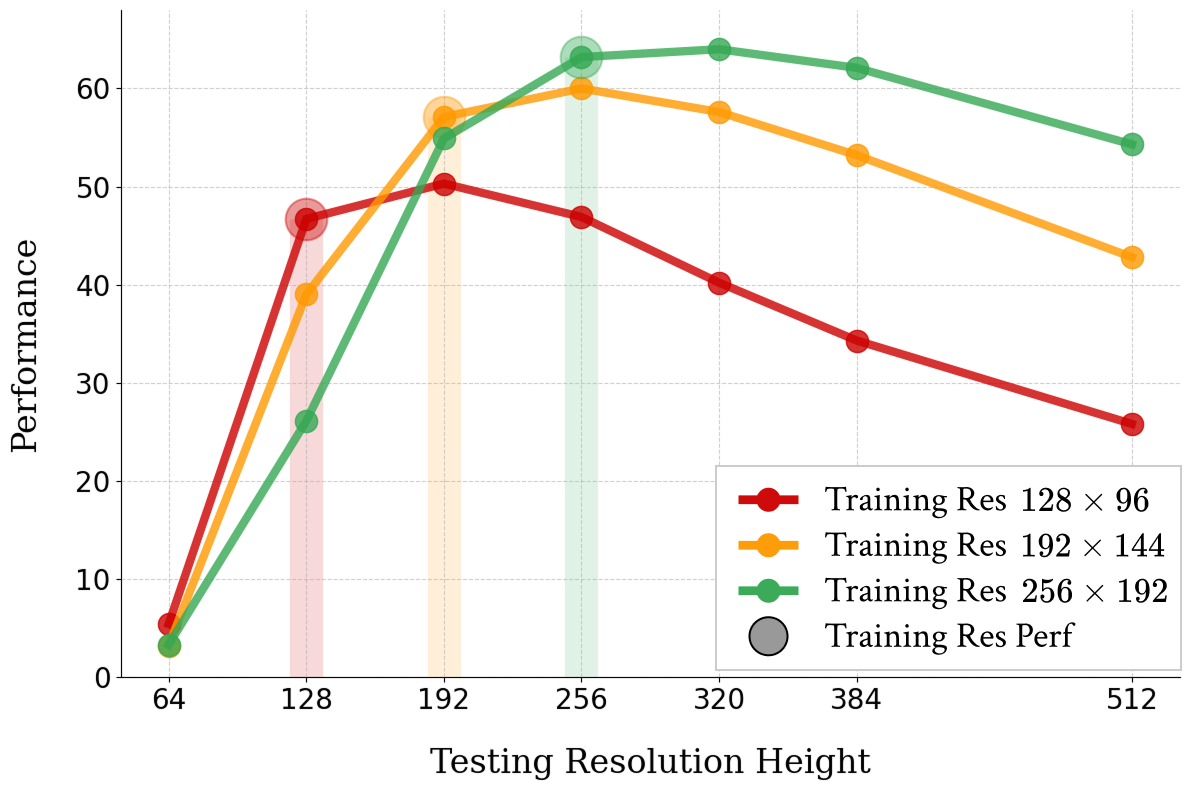}
    \caption{ For the Human Pose Estimation task in the centralized learning scenario, each model was identically trained on a single resolution across three clients and tested on various resolutions. ``Training Res Perf'' denotes the resolution used during model training. The x-axis shows only the height of the training resolution, and all displayed resolutions maintain a 4:3 height-to-width aspect ratio.}
    \label{fig:single_res_test}
    \vspace{-10pt}
\end{figure}

In realistic FL scenarios for high-resolution regression, clients inevitably possess images at different resolutions rather than at a single fixed size. As discussed in Section~\ref{sec:intro}, this variation introduces resolution heterogeneity, which, after model aggregation, can hinder FL by significantly degrading performance.

Figure~\ref{fig:single_res_test} shows that the ViT-based models suffer from critical limitations: 
We examined the performance variation when a ViT model, trained in a centralized manner at a single resolution, was tested at different resolutions.
The task was keypoint detection, a representative high-resolution regression task.
This involves estimating human keypoint locations from a single image and is commonly used for detailed human analysis. 
As shown in Figure~\ref{fig:single_res_test}, at its trained resolution, each model achieved peak performance, but its accuracy dropped sharply when tested on unseen resolutions. This indicates that the characteristics of the learned weights vary significantly depending on resolution.

\subsection{Resolution Drift: Performance Degradation in Federated Learning}
\label{subsec:resolution-drift}

As discussed in Section~\ref{subsec:risk_of_res-hetero}, resolution heterogeneity can introduce inconsistencies in the learned representations across clients who may possess data with different resolutions. These inconsistencies, when aggregated on the server, may significantly decrease the model's performance. To investigate this issue further, we conducted a controlled experiment to demonstrate explicitly how resolution heterogeneity can impair the effectiveness of federated learning.

\begin{table}[htbp]
\caption{Low-resolution ($128\times96$) test accuracy on the human pose estimation task for models trained on various resolution triplets in a federated learning scenario, illustrating the \textbf{resolution-drift} phenomenon. ``Res.'' denotes resolution.}
\label{tab:res_drift}
\centering
{
\begin{tabularx}{1.0\linewidth}{YYYY}
\toprule
\multicolumn{3}{c}{\textbf{Trained Resolution}} & \textbf{Test Res.}  \\ 
\cmidrule(lr){1-3}\cmidrule(l){4-4}
\textbf{Res.\,1} & \textbf{Res.\,2} & \textbf{Res.\,3} & $\mathbf{128\times96}$\\
\midrule
$\mathbf{128\times96}$ & $\mathbf{128\times96}$ & $\mathbf{128\times96}$ & \textbf{52.8} \\
$128\times96$           & $128\times96$           & $192\times144$       & 52.9 \\
$128\times96$           & $192\times144$          & $192\times144$       & 52.4 \\
$128\times96$           & $128\times96$           & $256\times192$       & 52.3 \\
$128\times96$           & $192\times144$          & $256\times192$       & 51.6 \\
$128\times96$           & $256\times192$          & $256\times192$       & 51.3 \\
\bottomrule
\end{tabularx}}
\end{table}

In particular, we simulated an FL setting in which each client was assigned a dataset with a distinct input resolution. Despite using the same model architecture and training procedures across all clients, the heterogeneity in resolution leads to diverging updates, rendering aggregation on the server more challenging. Table~\ref{tab:res_drift} presents the performance of the global models trained under resolution-diverse scenarios.

In Table~\ref{tab:res_drift}, the top row corresponds to all the clients using low-resolution data. Moving downward, the number of clients using high-resolution data increases.
Although higher-resolution data generally provide richer visual information and improve model learning, the inference performance on low-resolution inputs decreases as the training data diverge more from that resolution. Notably, only the configurations {$128\times96$, $128\times96$, $192\times144$} yielded a slight improvement, possibly because of the inclusion of more informative training samples.
The results in Table~\ref{tab:res_drift} reveal that federated models trained with heterogeneous resolutions consistently underperform compared to those trained in resolution-homogeneous settings. This performance drop, which we refer to as \textbf{ resolution drift }, highlights the tangible risk posed by resolution heterogeneity in FL environments. Such drift arises because the model is uncertain regarding to which resolution it should adapt, with degraded performance across all resolutions. This emphasizes that federated learning frameworks must explicitly account for resolution-related variability, and ensure robustness and generalization.

\section{Proposed Method: Resolution Adaptive Federated Learning (RAF)}
\label{sec:proposed_method}

\begin{figure*}[t]
    \centering    
    \includegraphics[width=0.95\textwidth]{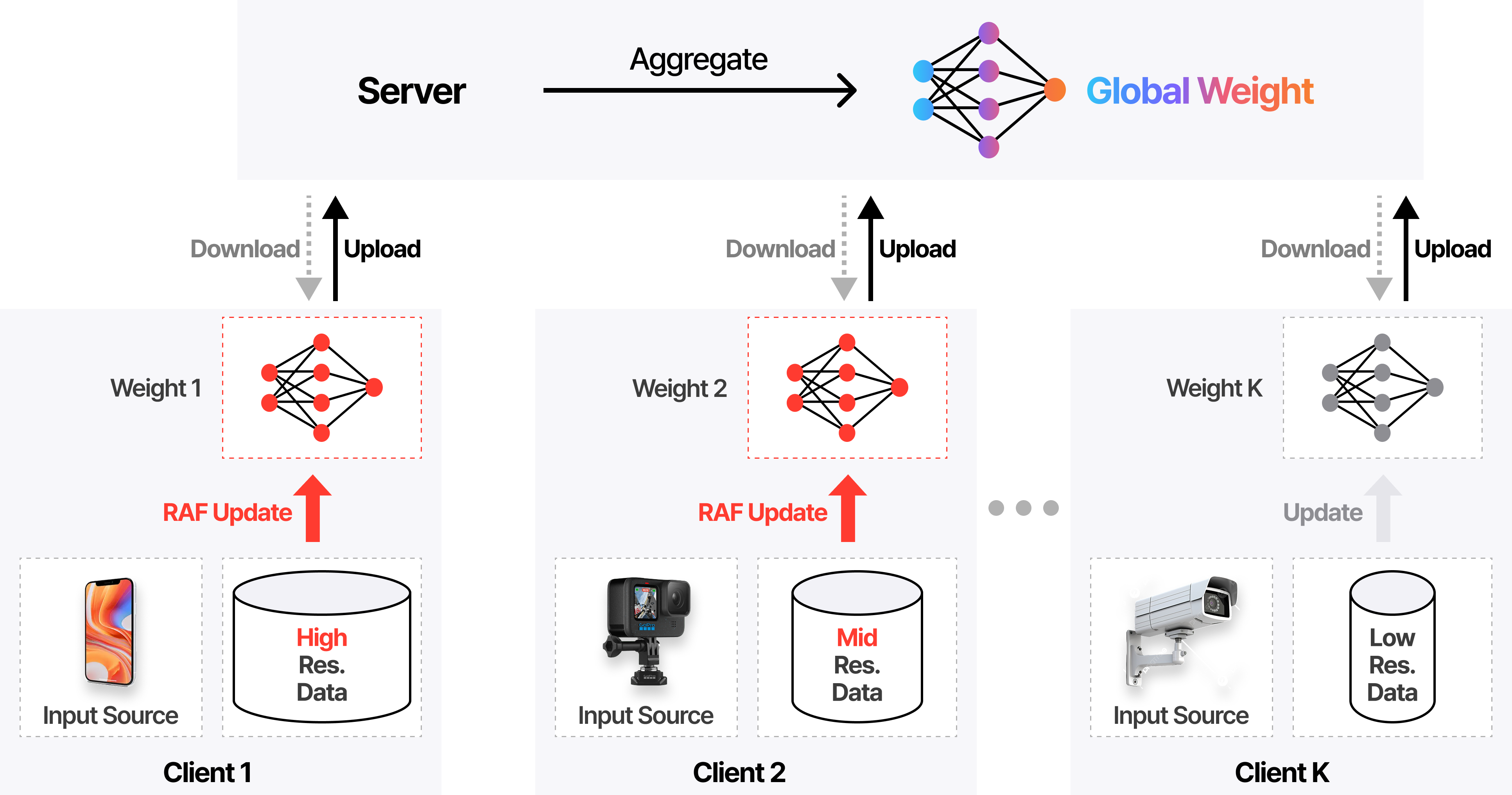}
    \caption{Overview of RAF when applied to clients with varying-resolution data. RAF fully exploits spatial information in multi-resolution images. ``Res.'' denotes resolution. RAF does not update at the lowest resolution, since further downsampling is impossible.}
    \label{fig:fl-overview}
    \vspace{-10pt}
\end{figure*}

Section~\ref{sec:pre-analysis} explains why federated learning in real-world settings must address high-resolution regression tasks and the resulting resolution-drift phenomenon is experimentally and visually demonstrated. When clients are trained at different image resolutions, their local updates overfit to those specific resolutions. Averaging these resolution-specific weights through FedAvg confuses the global model, causing its performance to decrease even though the input images contain rich spatial information. Consequently, clients have little incentive to participate. Why would they incur communication overhead if the aggregated model performs worse than the local training? To overcome this barrier, we propose a multi-resolution knowledge distillation framework that allows each client to leverage spatial features from its unique-resolution data and thereby effectively mitigate resolution-drift.

\subsection{Overview of RAF}

Figure~\ref{fig:fl-overview} shows a federated learning scenario in which multiple clients participate using data captured at different resolutions. As noted in Section~\ref{sec:intro}, it is highly unlikely that all clients in a real-world FL deployment will possess datasets with identical resolutions. For instance, images captured by a smartphone are often of high resolution, those captured by a small action camera are lower, and CCTV footage may be even lower.

To model this heterogeneity, we assume that each client $k$ holds a private dataset $D_{k}=\{(x_{k,j},T_{k,j})\}_{j=1}^{n_{k}}$
where \(x_{k,j}\in\mathbb{R}^{H_k\times W_k\times C}\) represents
\(j\)-th image recorded at the client's native resolution
\((H_k,W_k)\), and
\(T_{k,j}\) is a ground-truth heat-map defined on the same
pixel grid.
The overall training flow followed the standard FedAvg federated learning procedure.
Formally, let $N$ be the number of clients and suppose that the $k$-th client holds $n_{k}$ training samples $x_{k,j}$ (where $j=1,\;\dots,\;n_{k}$).
We seek a global objective model parameterized by $w$, which minimizes the aggregate objective.

\begin{align}
    \min_{w}\;\mathcal{L}(w) \;\triangleq\;\frac{1}{N}\sum_{k=1}^{N}\mathcal{L}_{k}(w).
\end{align}

In each communication round, the server broadcasts the current global weights $w$ to all clients. Each client subsequently performs local updates on its resolution-specific dataset \(\{D_{k},r_{k}\}\), where $r_{k}$ indicates the client's input resolution. Crucially, to prevent each client's model from overfitting to its own resolution, we augment the local training objective using a multi-resolution knowledge distillation term. This additional term encourages each client to incorporate spatial information from other resolutions, thereby counteracting the resolution drift.

After local training, each client uploads its updated weights to the server, which aggregates them using weighted averaging (FedAvg). The updated global model is then redistributed to all the clients, and the process is repeated until convergence. Detailed derivations of the local objective explicitly incorporating multi-resolution distillation are provided in Section~\ref{subsec:mrkd}.

\subsection{Maximally Utilizing Spatial Information via Multi-Resolution Knowledge Distillation}
\label{subsec:mrkd}

\begin{figure*}[t]
    \centering
    \includegraphics[width=0.95\textwidth]{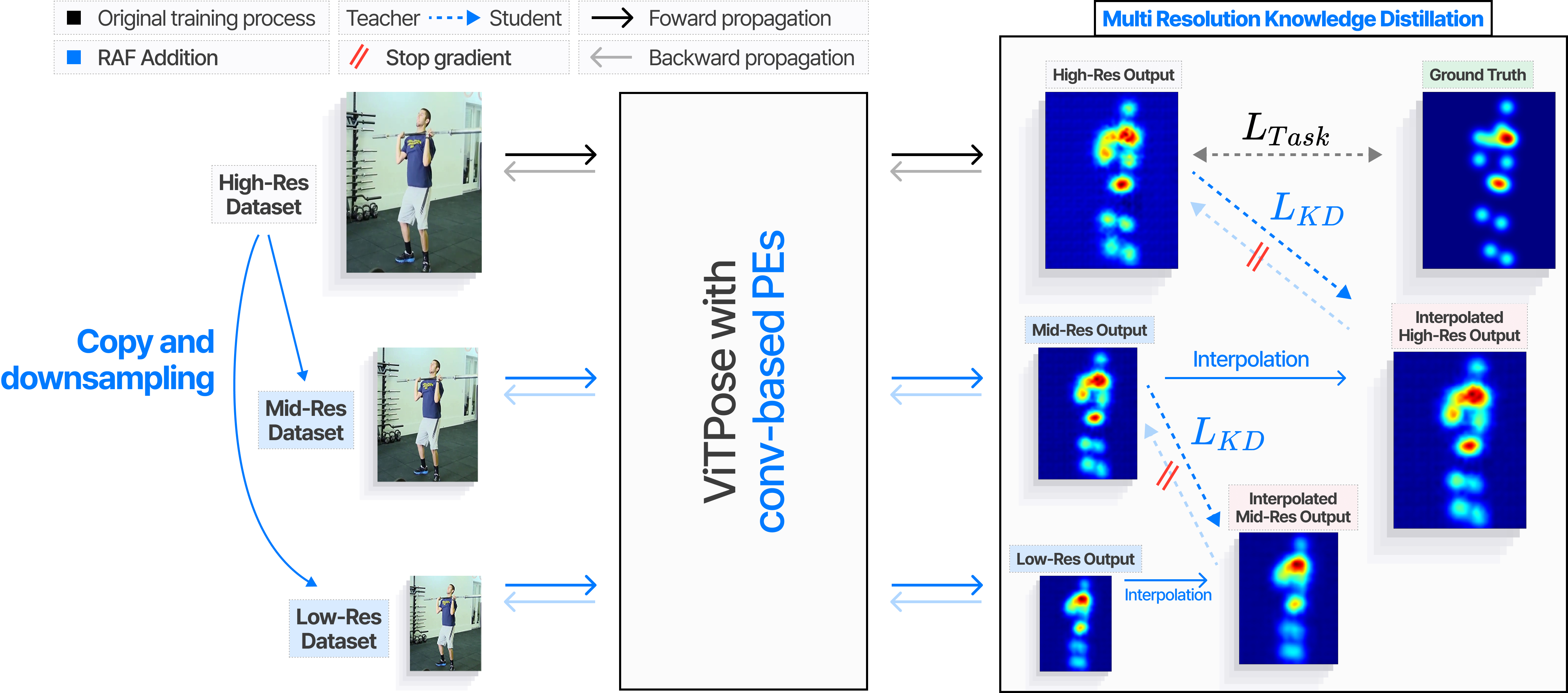}
    \caption{Detailed illustration of RAF local training on a high-resolution client. The client's original dataset is downsampled to create mid- and low-resolution inputs. The black arrows and boxes depict the standard ViTPose training flow, while the blue arrows and highlights indicate the additional RAF components. ``Stop gradient'' denotes that gradients are detached on the teacher branch during backpropagation for knowledge distillation.}
    \label{fig:mrkd}
    \vspace{-10pt}
\end{figure*}

In this subsection, we describe how our proposed model and training scheme effectively mitigate resolution drift. First, we outline the backbone architecture and its modifications to support multi-resolution inputs. We then introduce our Multi-Resolution Knowledge Distillation (MRKD) method and explain how it serves as a resolution-aware regularizer.

\subsubsection{Model Architecture}  
To leverage the strengths of transformer-based networks (as discussed in Section~\ref{subsec:vit}) and ensure compatibility with recent vision-transformer models, we adopt ViTPose~\cite{vitpose}, whose encoder follows the standard ViT architecture. Specifically, it consists of a patch-embedding layer (including positional embeddings), followed by several transformer blocks, each containing a multi-head self-attention layer and a feed-forward network layer.
Therefore, ViTPose relies on APE tied to a single resolution; it cannot accommodate multi-resolution training.

The ViT-based encoder requires a one-dimensional sequence of token embeddings rather than a two-dimensional image. 
The input 2D image $x\in\mathbb{R}^{H\times W\times C}$ is reshaped into a sequence of flattened 2D patches $x_{p}\in\mathbb{R}^{N_{p}\times(P^{2} C)}$, where $(H, W)$ is the resolution of the image, $(P,P)$ is the resolution of each image patch, $C$ is the number of channels, and $N_{p}=\frac{HW}{P^{2}}$ is the number of patches.
Next, a learnable linear projection $\mathbf{E}\in\mathbb{R}^{(P^{2} C)\times D}$ maps every patch to a $D$-dimensional latent vector to produce embedded patches $\mathbf{z}_{p}\in\mathbb{R}^{N_{p}\times D}$.
Then, a learnable positional embedding $\mathbf{E}_{\text{pos}}\in\mathbb{R}^{N_{p}\times D}$ is added to each token.

Note that the shape of $\mathbf{E}_{\text{pos}}$ depends directly on $N_{p}$, which, in turn, is determined by the input image resolution $(H,W)$. If the resolution changes, $N_{p}$ also changes, and the learned positional embedding $\mathbf{E}_{\text{pos}}$ must be reinitialized or resized. Consequently, a model trained with APE at one resolution cannot be directly generalized to another resolution, causing the accuracy to drop sharply on the unseen resolutions.
To prevent this and enable dynamic multi-resolution learning, we replace ViTPose's fixed APE with the convolution-based positional embeddings proposed in ResFormer~\cite{resformer}. First, we adopt a Global Positional Embedding (GPE) module immediately after the patch embedding layer. GPE employs a $3\times3$ depth-wise convolution to inject a smooth, global spatial context across the entire feature map. Next, within each multi-head self-attention block, we integrate a Local Positional Embedding (LPE) module, which also uses a $3\times3$ depthwise convolution, but focuses on capturing fine-grained local relationships. By substituting ViTPose's APE with GPE and LPE, our model can process inputs at arbitrary resolutions without modifying the network architecture.

\subsubsection{Multi-Resolution Knowledge Distillation}

\begin{algorithm}[htbp]
\caption{Resolution Adaptive Federated Learning (RAF)}
\label{alg:RAF}
\begin{algorithmic}[1]
\REQUIRE Number of rounds \(T\); number of local epochs \(E\); number of clients $N$; learning rate \(\eta\); distillation weight \(\alpha\); initial global model \(w_{0}\); each client $k$ holds private dataset $D_{k}$; number of local data samples available on that client $n_{k}$; total number of resolutions used in local training for $k$-th client $N_{k,\mathrm{res}}$
\FOR{round \(t = 0\) to \(T-1\)}
    \STATE Server broadcasts \(w_{t}\) to all clients
    \FOR{each client $k$ \textbf{in parallel}}        
        \STATE update local model parameter $w_{k}$ with $w_{t}$
        \STATE Let $D_{k}^{(i)}$ denote the dataset downsampled to the $i$-th resolution level, where larger $i$ corresponds to progressively lower resolutions.
        \FOR{$i=1$ to $N_{k,\mathrm{res}}-1$}
            \STATE Downsample entire $D_{k}^{(0)}$ to obtain $D_{k}^{(i)}$
        \ENDFOR
        \FOR{local epoch \(e=1\) to \(E\)}
            \FOR{each sample $(x_{k, j}^{(0)}, T_{k, j})\in D_{k}^{(0)}$}
                \STATE Select corresponding sample $x_{k, j}^{(i)}$ for $i=1,\;\dots,\;N_{k,\mathrm{res}}-1$
                \STATE \(\triangleright\) Compute task loss $\mathcal{L}_{k, \mathrm{task}}(w)$ as defined in Equation~\eqref{eq:loss_task}
                \STATE \(\triangleright\) Compute distillation loss $\mathcal{L}_{k, \mathrm{kd}}(w)$ as defined in Equation~\eqref{eq:loss_kd}
                \STATE \(\triangleright\) Compute total loss $\mathcal{L}_{k}(w)$ as defined in Equation~\eqref{eq:loss_total}
                \STATE \(\triangleright\) Backward and update
                \[
                  w_{k} \leftarrow w
                  - \eta\,\nabla_{w_{k}}\mathcal{L}_{k}(w)
                \]
            \ENDFOR
        \ENDFOR
        \STATE Client \(k\) sends \(w_{k}\) back to the server
    \ENDFOR
    \STATE Server aggregates
    \[
      w_{t+1} = \frac{1}{N}\sum_{k=1}^{N}w_{k}
    \]
\ENDFOR
\STATE \textbf{Output:} Final global model $w_{T}$
\end{algorithmic}
\end{algorithm}

Figure~\ref{fig:mrkd} illustrates the local training procedure for clients by using high-resolution data. Unlike standard ViTPose, which operates at a single fixed resolution, RAF requires multi-resolution inputs during training to achieve strong generalization across resolutions.

Let \(N_{k,\mathrm{res}}\) be the total number of
resolution levels employed when the client \(k\) undergoes local training.
For every original image
\(x_{k,j}\in\mathbb{R}^{H_{k}\times W_{k}\times C}\)
(\(j=1,\dots,n_{k}\)), we create
\(N_{k,\mathrm{res}}-1\) additional downsampled copies, denoted
\begin{align}
x_{k,j}^{(i)}\in\mathbb{R}^{H_{k}^{(i)}\times W_{k}^{(i)}\times C}\quad\bigl(i=0,\dots,N_{k,\mathrm{res}}-1\bigr),
\end{align}
where $x_{k,j}^{(0)}=x_{k,j}$ and, for $i>0$, $H_{k}^{(i)}<H_{k}^{(i-1)}$ and $W_{k}^{(i)}<W_{k}^{(i-1)}$ hold. Here, \(i\) is the resolution index (the larger the index, the lower the spatial resolution), and the subscripts \(k\) and \(j\) identify the client and the sample, respectively.
Each image \(x_{k,j}^{(i)}\) is obtained by interpolation of the native image \(x_{k,j}\) to the target size \(H_{k}^{(i)}\times W_{k}^{(i)}\).

After generating these $N_{k,\mathrm{res}}$ multi-resolution inputs, we pass each $x^{(i)}_{k,j}$ through model $\mathcal{M}$ parameterized by the local weight $w$ to obtain the corresponding heatmap output 
\begin{align}
    y^{(i)}_{k,j} \;=\; \mathcal{M}(x^{(i)}_{k,j};w).
\end{align}
First, we define a task loss $\mathcal{L}_{\text{task}}$ as the highest resolution output $y^{(0)}_{k,j}$. Following ViTPose~\cite{vitpose}, we use the mean squared error (MSE) against the ground-truth heatmap $T_{k, j}$:
\begin{align}
    \label{eq:loss_task}
    \mathcal{L}_{k,\text{task}}(w)
    = \frac{1}{n_k} \sum_{j=1}^{n_k} \left\| y^{(0)}_{k,j} - T_{k, j} \right\|_{2}^{2}.
\end{align}

To enforce \emph{scale-consistency} and encourage the model to extract rich spatial features from low-resolution inputs, we introduce a \emph{multi-resolution knowledge distillation} loss $\mathcal{L}_{\text{kd}}$. Specifically, we treat the heatmap output at resolution \(i-1\) as the ``teacher'' and the heatmap at resolution \(i\) as the ``student.'' By minimizing the MSE between these output pairs, we push the model to produce low-resolution heatmaps that resemble the higher-resolution ones. Formally,
\begin{align}
    \label{eq:loss_kd}
    \mathcal{L}_{k,\text{kd}}(w)
    = \sum_{i=1}^{N_{k,\mathrm{res}}-1} \frac{1}{n_k} \sum_{j=1}^{n_k} 
    \left\| \operatorname{sg}\left(y^{(i-1)}_{k,j}\right) - U_i^{i-1} y^{(i)}_{k,j} \right\|_2^2.
\end{align}
Note that during backpropagation, we detach the teacher output $\left(y^{(i-1)}_{k,j}\right)$ from the computational graph with $\operatorname{sg}$ operator (as in BYOL~\cite{byol}), which denotes the stop gradient, so that the gradients flow only through the student branch $y^{(i)}_{k,j}$. The matrix $U_i^{i-1} \in \mathbb{R}^{H_{k}^{(i-1)}W_{k}^{(i-1)} \times H_{k}^{(i)}W_{k}^{(i)}}$ is a fixed linear upsampling operator that maps low-resolution predictions at level $i$ to higher resolution $i-1$. 
Finally, the combined local objective for client $k$ is:
\begin{align}
    \label{eq:loss_total}
    \mathcal{L}_{k}(w)
    = \mathcal{L}_{k,\text{task}}(w) +\alpha \mathcal{L}_{k,\text{kd}}(w)+\gamma \mathcal{L}_{k,\text{reg}}(w),
\end{align}
where \(\alpha>0\) balances task accuracy and scale-consistency regularization, $\mathcal{L}_{k,\text{reg}}$ indicates an $\ell_{2}$ regularization, and $\gamma$ is its coefficient. 
Algorithm~\ref{alg:RAF} summarizes the RAF procedure. By replicating each original image $N_{k,\mathrm{res}} - 1$ times at progressively lower resolutions and applying knowledge distillation in a teacher-student hierarchy, RAF ensures that the model learns to generalize across all $N_{k,\mathrm{res}}$ resolutions, thus mitigating the resolution drift discussed earlier.

\subsubsection{MRKD as a Resolution Regularizer}
Our simple, yet effective MRKD approach acts as a resolution-aware regularizer. By aligning the teacher and student outputs across different resolutions, MRKD prevents the model from overfitting to any single resolution. This regularization effect improves generalization across all scales, thereby boosting the accuracy and robustness for both low and high unseen resolutions. 
In Section~\ref{subsec:low to high}, we show that MRKD-trained models can outperform their baselines even when evaluated on interpolated inputs at resolutions higher than those accessible during inference. This indicates that MRKD helps the network extract and utilize spatial information more effectively.

\subsection{Convergence Analysis}

This section establishes that the proposed RAF algorithm enjoys the same asymptotic convergence rate as standard \textsc{FedAvg}.
Our analysis concentrates on the \emph{late} phase of training, where a deep network is empirically observed, and theoretically justified~\cite{chizat2019lazy,lyu2020margin,papyan2020prevalence}, to behave like its linearisation.
Specifically, once the shared feature extractor has \emph{converged} (i.e., its weights evolve only marginally), the optimisation signal is
absorbed almost exclusively by the last affine layer.
We therefore regard the backbone up to the penultimate layer as a fixed feature map and study the remaining optimisation as that of a \emph{linear model with frozen features}, a viewpoint often called the \emph{lazy-training} or \emph{post-neural-collapse} regime.

Under the feature-converged assumption, we linearize the original network objective with respect to the final-layer weights; the resulting expected local loss for client \(k\) at round \(t\) is given by
{\small\begin{align}
&\mathcal{L}_{k}(w) 
= \frac{1}{n_k} \sum_{j=1}^{n_k} \left\| (\psi_{k,j}^{(0)})^\top w - T_{k, j} \right\|_2^2 \\
& + \alpha \sum_{i=1}^{N_{k,\mathrm{res}}-1} \frac{1}{n_k} \sum_{j=1}^{n_k} 
\left\| (\psi_{k,j}^{(i-1)})^\top w_t - U_i^{i-1} (\psi_{k,j}^{(i)})^\top w \right\|_2^2 \\
& + \frac{\gamma}{2} \|w\|_2^2
\end{align}}
Throughout this section we adopt the following concise notation.
For client \(k\) and sample \(j\) the expression \((\psi_{k,j}^{(i)})^{\top}w\)
(\(i=0,\dots,N_{k,\mathrm{res}}-1\)) denotes the heat-map obtained by projecting the feature vector
\(\psi_{k,j}^{(i)}\) with the weight of the last layer, $w$.
In late rounds, the backbone is assumed \emph{feature-converged}, so only the last-layer weight is updated.

This linearized objective includes both a supervised loss term for the highest-resolution predictions and a multi-resolution knowledge distillation loss that encourages consistency between consecutive resolutions through the fixed upsampling operators \(U_i^{i-1}\).
We now introduce an alternative formulation of the loss:
{\small\begin{align}
\bar{\mathcal L}_{k}(w)
=
\frac{1}{n_k}\sum_{j=1}^{n_k}&
        \bigl\|(\psi_{k,j}^{(0)})^{\top}w - T_{k,j}\bigr\|_{2}^{2}\\
+
\alpha
\sum_{i=1}^{N_{k,\mathrm{res}}-1}\!&
\frac{1}{n_k}\sum_{j=1}^{n_k}
        w^{\top}(\psi_{k,j}^{(i)})^{\top}(U_{i}^{i-1})^{\top}\nonumber\\
        &\times
        \Bigl[(\psi_{k,j}^{(i-1)})^{\top}
              - U_{i}^{i-1}(\psi_{k,j}^{(i)})^{\top}\Bigr]w \\
+
\frac{\gamma}{2}\,\|w\|_{2}^{2}&.
\end{align}}
Although \(\mathcal{L}_{k}(w)\) and \(\bar{\mathcal{L}}_{k}(w)\) differ globally, they are locally equivalent at the current iterate \(w = w_t\). Specifically,
\begin{align}
\mathcal{L}_{k}(w_{t}) &= \bar{\mathcal{L}}_{k}(w_t), \\
\nabla_w \mathcal{L}_{k}(w_{t}) &= \nabla_w \bar{\mathcal{L}}_{k}(w_t).
\end{align}
Therefore, gradient-based optimization of \(\mathcal{L}_{k}(w)\) in any neighbourhood of the current iterate \(w_t\) yields exactly the same update direction as would be obtained from \(\bar{\mathcal{L}}_{k}(w)\).  Both objectives thus attain the same critical point, and their local trajectories are indistinguishable.  Moreover, because the alternative loss \(\bar{\mathcal{L}}_{k}(w)\) is a quadratic form whose Hessian is augmented by the positive-definite matrix \(\tfrac{\gamma}{2}I_d\) with \(\gamma>0\), it is strictly convex; consequently, it possesses a unique and non-trivial minimiser \(w^{\star}\) rather than a degenerate solution.
Hence, the two formulations induce identical convergence behaviour, and \(\bar{\mathcal{L}}_{k}(w)\) can legitimately be used as a surrogate objective in theoretical analysis.

To analyze the convergence property, we first introduce the following assumptions.  
\begin{assumption}[Boundedness]\label{asm:bounded}
There exist positive constants \(M_{\phi}\), \(M_{U}\), and \(M_{T}\) such that, for any client $k$, for all resolution level \(i\) and every sample \(j\) pairs,
\begin{align}
  \|\psi_{k,j}^{(i)}\|_{2} \le M_\phi,\;\|U_{i}^{i-1}\|_{2}\leq M_{U},\;\|T_{k, j}\|_{2} \le M_T.
\end{align}
\end{assumption}
\begin{assumption}[Unbiased Stochastic Gradients]\label{asm:stoch-grad}
For any client \(k\) and parameter vector \(w\) with \(\|w\|_{2}\le R\),
let \(\xi\) denote a random index drawn uniformly from the local sample set.
The sample gradient
\(
\nabla \mathcal L_{k}(w;\xi)
\)
satisfies $\mathbb E_{\xi}\!\bigl[\nabla \mathcal L_{k}(w;\xi)\bigr]
= \nabla \mathcal L_{k}(w)$.
\end{assumption}
Note that Assumption \ref{asm:stoch-grad} is generally used in convergence analysis of federated learning \cite{fedavg_convergence}.
Under these conditions we shall prove three auxiliary propositions that bound (i) the gradient's Lipschitz constant, (ii) the objective's strong-convexity modulus, and (iii) the variance of stochastic gradients.
Taken together, these propositions deliver the smoothness, strong-convexity, and bounded-variance conditions required for the FedAvg convergence result of \cite{fedavg_convergence}.

\begin{proposition}[Smoothness of $\bar{\mathcal L}_{k}$]\label{prop:smooth}

Under Assumption~\ref{asm:bounded}, the gradient of $\bar{\mathcal L}_{k}(w)$ with respect to $W$ is $L$-Lipschitz with $L=2\bigl(1+\alpha(r-1)(M_{U}+1)^{2}\bigr)M_{\phi}^{2} + \gamma$.
\end{proposition}
\begin{proof}
We compute the gradient of each term in $\bar{\mathcal L}_{k}(w)$. 
Let us denote the gradient of the first term as
\begin{align}
\nabla_{w}\mathcal L_{k,\text{task}}(w)
    &=\frac{2}{n_k}\sum_{j=1}^{n_k}
      \psi_{k,j}^{(0)}\bigl((\psi_{k,j}^{(0)})^{\top}w - T_{k,j}\bigr).
\end{align}
This is a linear function of $w$, hence, its Lipschitz constant is computed as
\begin{align}
    L_{k,\text{task}} = \frac{2}{n_{k}} \sum_{j=1}^{n_{k}} \|\psi_{k,j}^{(0)}\|_{2}^2\leq 2M_{\phi}^{2}.
\end{align}

Next, let us find the Lipschitz constant of the gradient of knowledge distillation term.
Set $A_{ij}=U_{i}^{\,i-1}(\psi_{k,j}^{(i)})^{\top}$ and $B_{ij}=(\psi_{k,j}^{(i-1)})^{\top}$.
Then, the gradient of the knowledge distillation term is written as
\begin{align}
\nabla_{w}\mathcal L_{k,\mathrm{kd}}(w)
=\sum_{i=1}^{r-1}\frac{2}{n_k}\sum_{j=1}^{n_k}
      A_{ij}^{\top}\bigl(A_{ij}-B_{ij}\bigr)w.
\end{align}
The matrix norms satisfy $\|A_{ij}\|_{2}\le M_{U}M_{\phi}$ and
$\|B_{ij}\|_{2}\le M_{\phi}$, hence
$\|A_{ij}-B_{ij}\|_{2}\le M_{\phi}(M_{U}+1)$.  Consequently,
\begin{align}
&\bigl\|\nabla_{w}\mathcal L_{k,\mathrm{kd}}(w)
      -\nabla_{w}\mathcal L_{k,\mathrm{kd}}(v)\bigr\|_{2}
\\
&\le\frac{2}{n_k}\sum_{i=1}^{r-1}\sum_{j=1}^{n_k}
      \|A_{ij}^{\top}\|_{2}\,\|A_{ij}-B_{ij}\|_{2}\,\|w-v\|_{2}\notag\\
&\le 2(r-1)M_{U}M_{\phi}\,M_{\phi}(M_{U}+1)\|w-v\|_{2}.
\end{align}
The bound $M_{U}(M_{U}+1)\le (M_{U}+1)^{2}$ gives
\begin{align}
L_{k,\mathrm{kd}}\le 2(r-1)(M_{U}+1)^{2}M_{\phi}^{2}.
\end{align}
Finally, for the regularization term, one has
\begin{align}
\nabla_{w}\mathcal L_{k,\mathrm{reg}}(w)=\gamma w,
\qquad
L_{k,\mathrm{reg}}=\gamma.
\end{align}

Combining all Lipschitz constants, the global Lipschitz constant of $\bar{\mathcal L}_{k}$ is obtained as
\begin{align}
L_{k} &= L_{k,\mathrm{task}} + \alpha L_{k,\mathrm{kd}} + L_{k,\mathrm{reg}} \notag\\
  &\le 2M_{\phi}^{2} + \alpha\bigl[2(r-1)(M_{U}+1)^{2}M_{\phi}^{2}\bigr] + \gamma \notag\\
  &= 2\bigl(1+\alpha(r-1)(M_{U}+1)^{2}\bigr)M_{\phi}^{2} + \gamma.
\end{align}
This completes the proof of \(L\)-smoothness.
\end{proof}

\begin{proposition}[Strong convexity of $\bar{\mathcal L}_{k}$]\label{prop:strong}
Under Assumption \ref{asm:bounded}, $\bar{\mathcal L}_{k}$ is $\gamma$-strongly convex with respect to $w$.
\end{proposition}
\begin{proof}
From the regularization term in $\bar{\mathcal L}_{k}$, $\nabla^2 \bar{\mathcal L}_{k}(W) \succeq \gamma I$ holds, and the function is $\gamma$-strongly convex.
\end{proof}

\begin{proposition}[Bounded Gradient Norm and Variance]\label{prop:grad-var}
Let Assumption~\ref{asm:bounded} hold and assume the current iterate satisfies
\(\|w_{t}\|_{2}\le R\).
Define the constant $C=\mathcal{O}\left(M_{\phi}M_{T} +(rM_{\phi}^{2}M_{U}^{2}+\gamma) R\right)$. Then,
\begin{align}
  &\|\nabla \bar{\mathcal L}_{k}(w_{t})\|_2 \le C,\\
  &\mathbb{E}_{\xi} \left[ \left\| \nabla \bar{\mathcal L}_{k}(w_{t};\xi) - \nabla \bar{\mathcal L}_{k}(w_{t}) \right\|_2 \right] \le C.
\end{align}
\end{proposition}
\begin{proof}
We analyze the gradient of the local loss \( \bar{\mathcal L}_{k}(w) \), which consists of three terms: the task loss, the distillation loss, and the regularization term. Each will be bounded separately.

First, the gradient of the task loss is given by
\begin{align}
\nabla_{w}\mathcal L_{k,\text{task}}(w_t)
=
\frac{2}{n_k}\sum_{j=1}^{n_k}
\psi_{k,j}^{(0)}
\bigl((\psi_{k,j}^{(0)})^{\top}w_t - T_{k,j}\bigr).
\end{align}
Using the triangle and Cauchy-Schwarz inequalities, the norm can be bounded as
\begin{align}
&\bigl\|\nabla_{w}\mathcal L_{k,\text{task}}(w_t)\bigr\|_2
=
\Bigl\|
\frac{2}{n_k}\sum_{j=1}^{n_k}
        \psi_{k,j}^{(0)}
        \bigl((\psi_{k,j}^{(0)})^{\top}w_t - T_{k,j}\bigr)
\Bigr\|_{2}\\
&\le
\frac{2}{n_k}\sum_{j=1}^{n_k}
        \|\psi_{k,j}^{(0)}\|_{2}\,
        \bigl\|(\psi_{k,j}^{(0)})^{\top}w_t - T_{k,j}\bigr\|_{2}\\
&\le
\frac{2}{n_k}\sum_{j=1}^{n_k}
        \bigl(
          \|\psi_{k,j}^{(0)}\|_{2}^{2}\,\|w_t\|_{2}
          + \|\psi_{k,j}^{(0)}\|_{2}\,\|T_{k,j}\|_{2}
        \bigr)\\
&\le
2M_\phi\bigl(M_\phi R + M_T\bigr),
\end{align}
where we have used the bounds \( \|\psi_{k,j}^{(0)}\|_2 \le M_\phi \), \( \|T_{k, j}\|_2 \le M_T \), and \( \|w_{t}\|_2 \le R \).

Next, we bound the gradient of the distillation loss. Each individual term satisfies
\begin{align}
&\bigl\|
  \bigl(U_{i}^{i-1}\psi_{k,j}^{(i)\!\top}\bigr)^{\!\top}
  \bigl(
     U_{i}^{i-1}\psi_{k,j}^{(i)\!\top}w_{t}
     - \psi_{k,j}^{(i-1)\!\top}w_{t}
  \bigr)
\bigr\|_{2}\\
&\;\le\;
  \|\psi_{k,j}^{(i)}\|_{2}\,\|U_{i}^{i-1}\|_{2}\nonumber\\
  &\times  \Bigl(
     \|U_{i}^{i-1}\|_{2}\,\|\psi_{k,j}^{(i)}\|_{2}\,\|w_{t}\|_{2}
     + \|\psi_{k,j}^{(i-1)}\|_{2}\,\|w_{t}\|_{2}
  \Bigr) \\
&\;\le\;
  M_\phi M_U\bigl(M_U M_\phi R + M_\phi R\bigr) \\
&\;\le\;
  M_\phi^{2}(M_U+1)^{2}R.
\end{align}

Summing over all indices, the gradient norm of the distillation loss is bounded as
\begin{align}
\|\nabla_w \mathcal{L}_{k,\mathsf{kd}}(w)\|_2 
\le 2 (r - 1) M_\phi^2 (M_U + 1)^2 R.
\end{align}
Finally, the norm of the gradient of the regularization term is bounded by
$\|\gamma w\|_2 \le \gamma R$.
Combining all three bounds, the total gradient norm is upper-bounded by
\begin{align}
\|\nabla_w \bar{\mathcal{L}}_{k}(w)\|_2 
\le 2 M_\phi (M_\phi R + M_T) \\
+ 2\alpha(r - 1) M_\phi^2 (M_U + 1)^2 R + \gamma R.
\end{align}
This expression defines an upper bound constant, which satisfies
\begin{align}
\|\nabla_w \bar{\mathcal{L}}_{k}(w)\|_2 \le \mathcal{O}(M_\phi M_T + \alpha r M_\phi^2 M_U^2 R + \gamma R).
\end{align}

Finally, for the variance of the stochastic gradients, we note that each stochastic sample gradient \( \nabla \bar{\mathcal{L}}_{k}(w; \xi) \) is formed from a single summand and is thus bounded similarly to the full gradient. Applying Jensen's inequality, we obtain
\begin{align}
\mathbb{E}_\xi \left[ \left\| \nabla \bar{\mathcal{L}}_{k}(w;\xi) - \nabla \bar{\mathcal{L}}_{k}(w) \right\|_2 \right] \le 2 C,
\end{align}
which is also bounded by $C$ up to scalar scale.
This completes the proof.
\end{proof}

\begin{theorem}[Global Convergence of \textsc{raf}]\label{thm:raf}
Choose a stepsize sequence
\begin{align}
\eta_t \;=\;
\Theta\!\Bigl(
\frac{1}{\,\gamma\bigl(E+\alpha r M_U^{2}M_\phi^{2}/\gamma+t\bigr)}
\Bigr)\!,    
\end{align}
so that, once \(t\) exceeds the threshold  
\(E+\alpha r M_U^{2}M_\phi^{2}/\gamma\), the rule satisfies  
\(\eta_t=\Theta(1/(\gamma t))\).
Then after \(T\) communication rounds the FedAvg optimality gap obeys
\begin{align}
\mathbb{E}[\mathcal{L}(w_{T})] - \min_{w}\mathcal{L}(w)\leq \;
O\!(1/T).
\end{align}
\end{theorem}
\begin{proof}
Combining Propositions~\ref{prop:smooth} to \ref{prop:grad-var}, we identify the exact constants required by the
FedAvg analysis of \cite{fedavg_convergence}.
With these choices, all four technical conditions of
\cite{fedavg_convergence} are satisfied automatically.
See Propositions~\ref{prop:smooth} to~\ref{prop:grad-var}.
\end{proof}
Under the explicit constants derived for \textsc{RAF}, Theorem~\ref{thm:raf}
recovers the classical FedAvg rate.  
Because the multi-resolution distillation term is convex and
Lipschitz-smooth, the presence of heterogeneous resolutions does not
deteriorate the convergence guarantee, while it does enhance robustness
across multiple input scales.

\section{Experiments}
\label{sec:experiments}
\subsection{Experimental Setup}
We evaluated the robustness of the proposed algorithm under resolution drift using a representative high-resolution regression task named Human Pose Estimation (HPE)~\cite{simple_baseline}. This involves localizing keypoints corresponding to human body joints in a person-centric input image. It is formulated as a high-resolution heatmap regression problem, in which the model outputs a set of spatial heatmaps, with each channel corresponding to a specific joint, and encodes the likelihood of that joint appearing at each spatial location. The final joint coordinates are obtained by independently identifying the (x, y) position with the maximum value in each output channel, which results in a set of joint locations equal to the number of target joints.

The MPII dataset~\cite{mpii}, which is a standard benchmark for human pose estimation, was used in the experiments. This provides predefined training and validation datasets. In our federated learning setup, each client was configured to hold 4K images sampled separately from the training data, and by default, the trained global model was evaluated using 1.5K validation images unless otherwise specified in the experimental section.
As a baseline, all the clients used the same model architecture, following the ViTPose~\cite{vitpose} configuration with a small Vision Transformer (ViT-S) backbone. The models were aggregated using the simple FedAvg. Our proposed method builds on this baseline architecture by incorporating a multi-resolution distillation mechanism for each client model. 
In the loss function, we set weighting coefficients to $\alpha=1$, and $\gamma=0.01$.
The AdamW optimizer was employed to train the models with batch sizes of 32 for training and 64 for testing. The initial learning rate was set to $2.5 \times 10^{-4}$, and the other settings, such as the learning rate decay schedule and data augmentation protocol followed those described in ViTPose~\cite{vitpose}.

\begin{figure}[t]
    \centering    
    \includegraphics[width=0.94\linewidth]{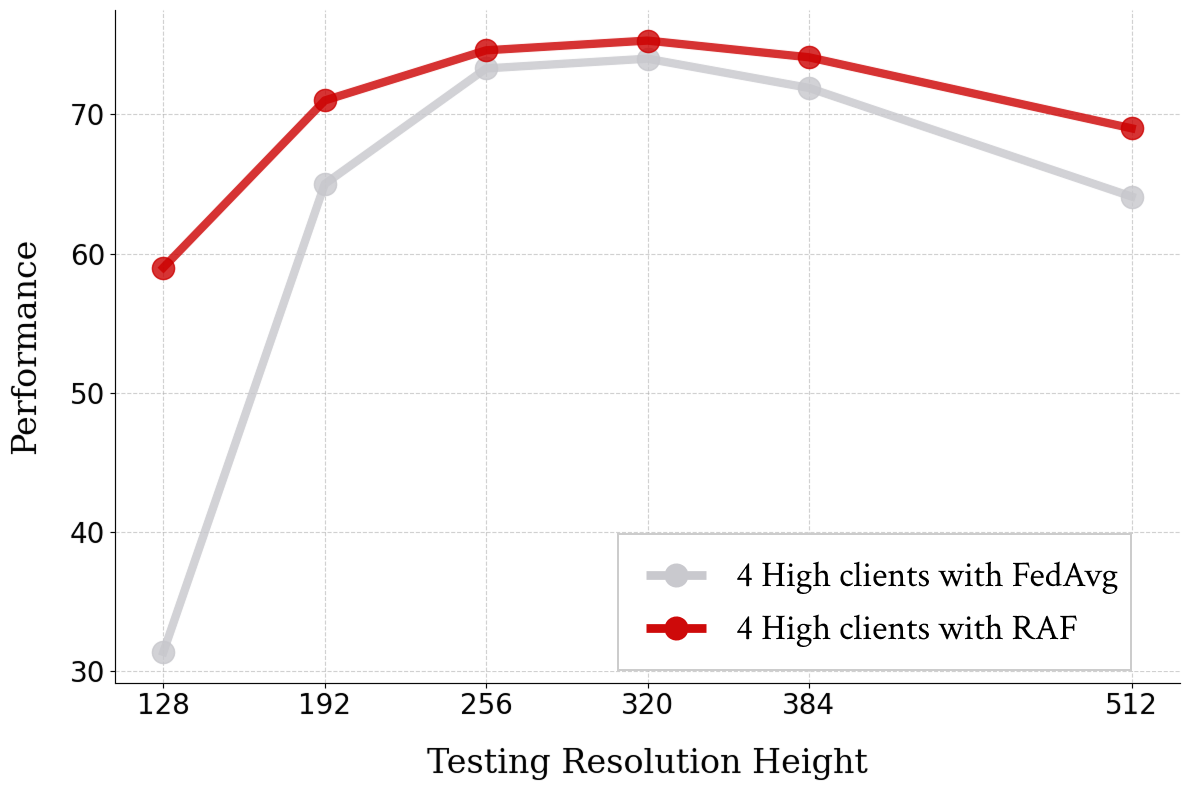}
    \caption{Inference accuracy on the human pose estimation task at various resolutions when four clients each hold only high-resolution ($256\times192$) datasets in a federated learning setting. The gray curve denotes the baseline using FedAvg aggregation, and the red curve denotes RAF. The x-axis indicates the height of each inference resolution; all resolutions maintain a 4:3 height-to-width aspect ratio.}
    \label{fig:h4-comp}
    \vspace{-15pt}
\end{figure}

\begin{figure}[t]
    \centering
    \begin{subfigure}[b]{0.32\linewidth}
        \includegraphics[width=\textwidth]{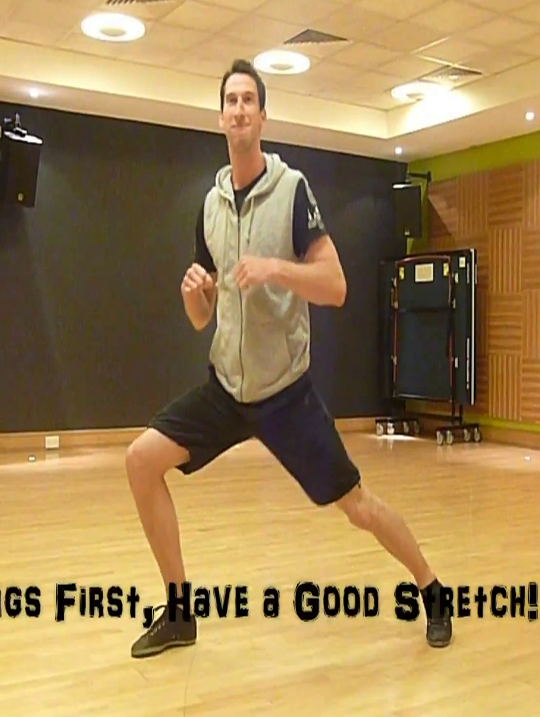}
        \caption{Input image}
        \label{fig:input}
        \end{subfigure}
    \hfill
    \begin{subfigure}[b]{0.32\linewidth}
        \includegraphics[width=\textwidth]{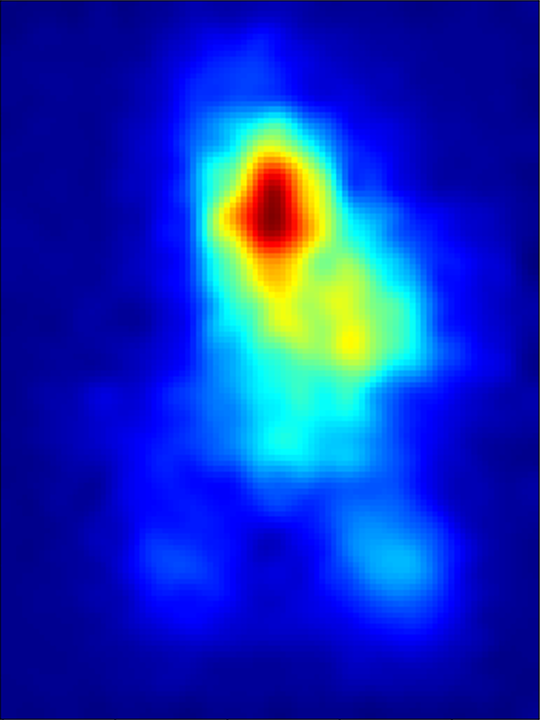}
        \caption{FedAvg}
        \label{fig:h4-standard}
        \end{subfigure}
    \hfill
    \begin{subfigure}[b]{0.32\linewidth}
        \includegraphics[width=\textwidth]{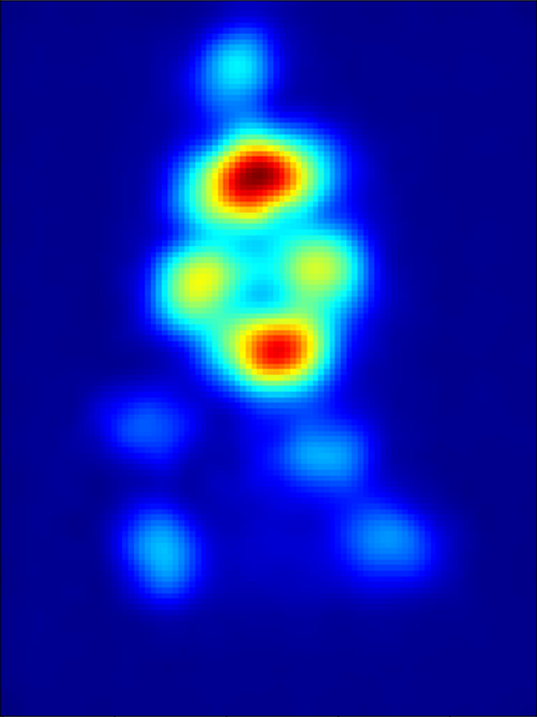}
        \caption{RAF}
        \label{fig:h4-raf}
        \end{subfigure}    
    \caption{Heatmap inference comparison on low-resolution ($128\times 96$) images. Figure~\ref{fig:input} is the input; Figure~\ref{fig:h4-standard} visualizes the heatmap inferred using weights obtained via FedAvg only from four high-resolution clients.; Figure~\ref{fig:h4-raf} is a visualization of the heatmap inferred using weights obtained via RAF from four high-resolution clients.}
    \label{fig:heatmap_comp}
    \vspace{-15pt}
\end{figure}

\subsection{Evaluation of Generalization Beyond Training Resolutions}
In this section, we demonstrate that our RAF method effectively mitigates the intrinsic limitations of ViT-based models, as discussed in Section~\ref{subsec:risk_of_res-hetero} and~\ref{subsec:resolution-drift}. 

Figure~\ref{fig:h4-comp} illustrates an FL scenario in which four clients are trained exclusively on high-resolution ($256\times192$) data. To assess robustness against resolution drift, we compared our proposed RAF method with the FedAvg-only baseline by evaluating both models across a spectrum of inference resolutions.
As shown, even when all client training data comprised only a single high resolution, RAF achieved substantially better accuracy at both lower and higher inference resolutions. Specifically, the baseline model suffered severe performance degradation to $128 \times 96$ and $192 \times 144$, whereas RAF yielded improvements of 27.6\% and 6.0\%.
These results indicate that RAF enables the model to experience multiple resolutions beyond the input resolution of each client by leveraging knowledge distillation with lower-resolution inputs. This afforded a balanced performance across all tested resolutions, including those not seen during training, even when all client input resolutions were the same.

Although Figure~\ref{fig:h4-comp} highlights the strength of our method, we reinforce this insight using visualized inference heatmaps. Figure~\ref{fig:heatmap_comp} compares the heatmaps produced by the FedAvg baseline trained on the four high-resolution clients with those from our RAF-trained model. At low resolution, the baseline model's heatmaps blur together, making it difficult to localize individual joints, whereas the RAF model clearly delineates each joint even under the same downsampling, demonstrating its superior robustness.

\subsection{Evaluation of the Regularization Effects of RAF}
\label{subsec:raf_comparison}
\begin{table*}[t]
\centering
\caption{Comparison of federated learning performance among three clients holding datasets with resolutions of $256\times192$, $192\times144$, and $128\times96$. Base and RAF variants are distinguished by whether FedAvg or FedProx is used for aggregation. Red numbers indicate the improvement of ``RAF (FedAvg)'' over ``Base (FedAvg)'', and blue numbers indicate the improvement of ``RAF (FedProx)'' over ``Base (FedProx)''.}
{\label{tab:RAF_comparison}
\begin{tabularx}{0.9\linewidth}{>{\raggedright\arraybackslash}p{0.2\linewidth}YYYYY}
\toprule
\textbf{Inference Resolution} & \textbf{Base (FedAvg)} & \textbf{Base (FedProx)} & \textbf{RAF (FedAvg)} & \textbf{RAF (FedProx)} \\
\midrule
$128\times96$ (Seen)       & 51.8  & 51.8 & \textbf{57.2 }\textcolor{red}{(+5.4)} & \textbf{57.1 }\textcolor{blue}{(+5.3)}\\
$192\times144$ (Seen)       & 64.6  & 65.0 & \textbf{67.1 }\textcolor{red}{(+3.5)} & \textbf{67.0 }\textcolor{blue}{(+2.0)}\\
$256\times192$ (Seen)       & 68.5  & 68.5 & \textbf{69.5 }\textcolor{red}{(+1.0)} & \textbf{69.7 }\textcolor{blue}{(+1.2)}\\
$320\times240$ (Unseen)     & 69.0  & 68.7 & \textbf{69.9 }\textcolor{red}{(+0.9)} & \textbf{70.1 }\textcolor{blue}{(+1.4)}\\
$384\times288$ (Unseen)     & 67.5  & 66.6 & \textbf{69.1 }\textcolor{red}{(+1.6)} & \textbf{69.0 }\textcolor{blue}{(+2.4)}\\
$512\times384$ (Unseen)     & 60.7  & 59.5 & \textbf{64.6 }\textcolor{red}{(+3.9)} & \textbf{64.4 }\textcolor{blue}{(+4.9)}\\
\bottomrule
\end{tabularx}}
\end{table*}

We evaluated RAF under resolution-drift conditions by comparing four configurations: Base (FedAvg), Base (FedProx), RAF (FedAvg), and RAF (FedProx). In each case, the three clients held high-($256\times192$), medium-($192\times144$), and low ($128\times96$) resolution data. For the RAF settings, we applied our multi-resolution distillation to each client's local model, leaving all other aspects identical to the corresponding baseline. This design isolates RAF's impact on mitigating resolution drift. The results are summarized in Table~\ref{tab:RAF_comparison}.

\begin{figure*}[b]
  \centering
  \begin{subfigure}[b]{0.31\textwidth}
    \includegraphics[width=\textwidth]{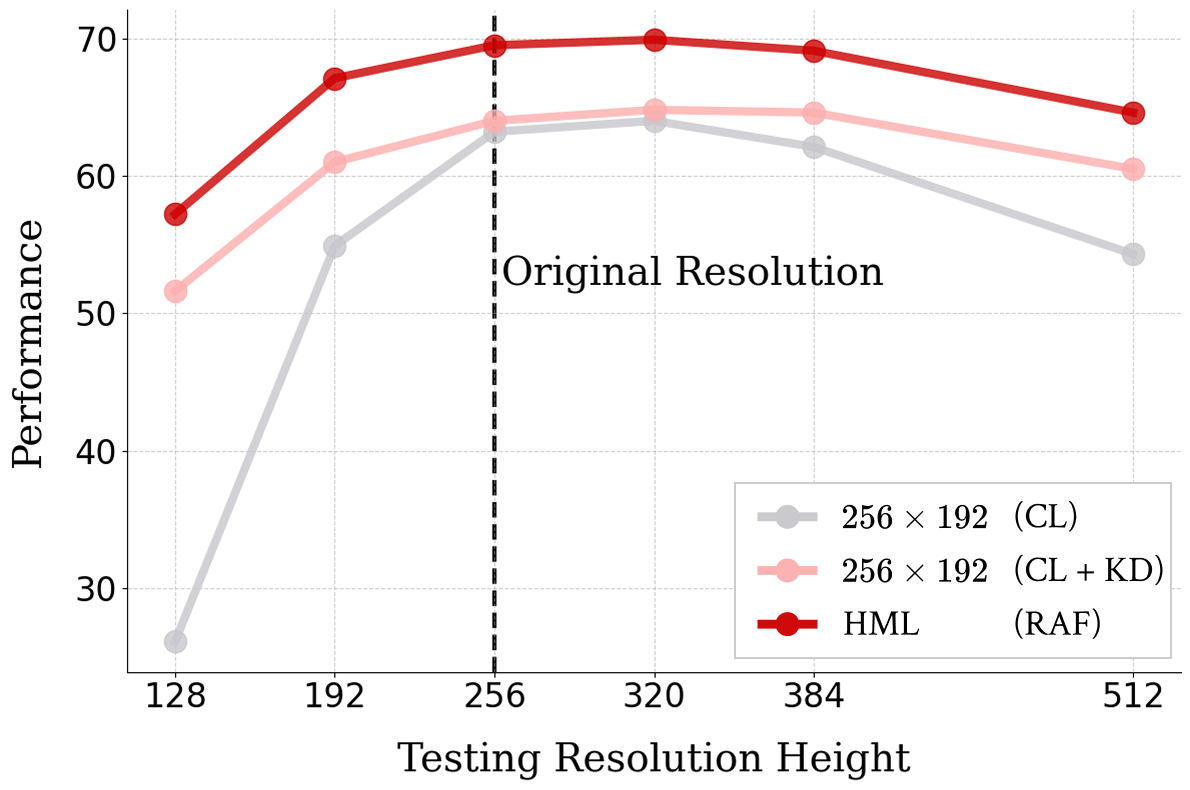}
    \caption{High Client Benefit}
    \label{fig:high-bene}
  \end{subfigure}
  \hspace{0.02\textwidth}
  \begin{subfigure}[b]{0.31\textwidth}
    \includegraphics[width=\textwidth]{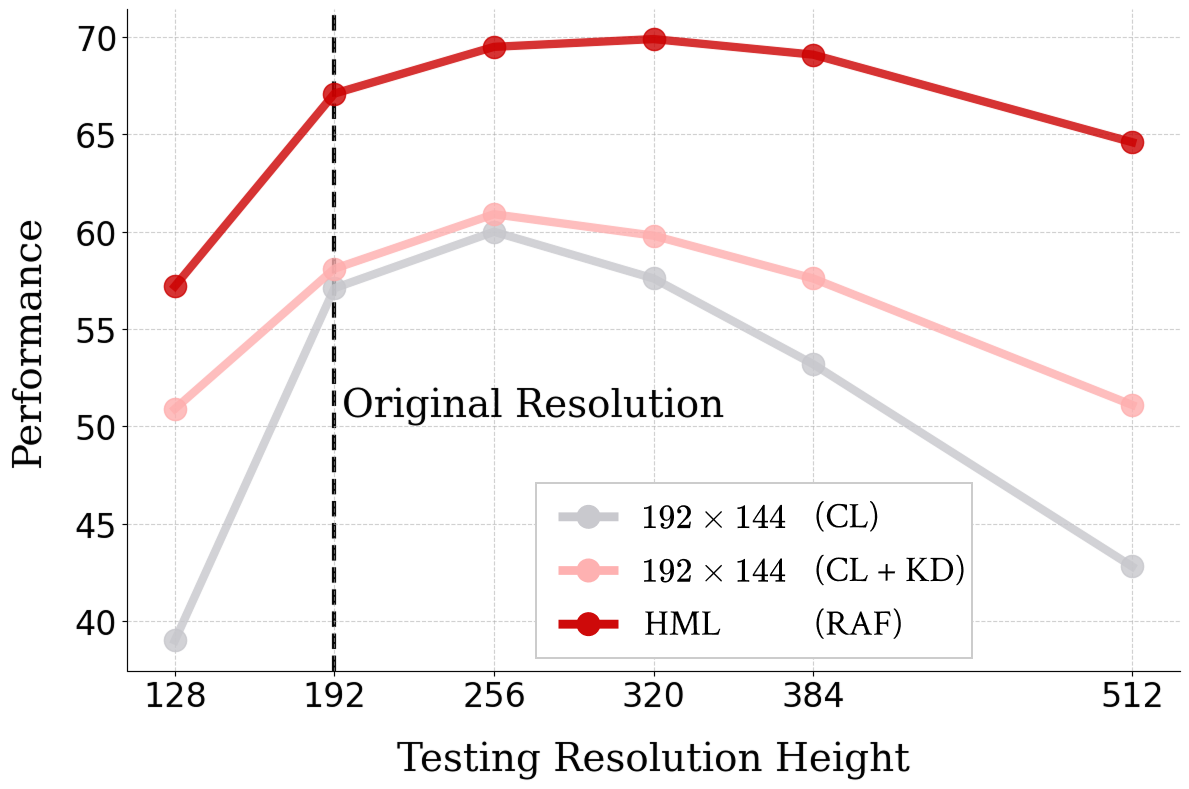}
    \caption{Mid Client Benefit}
    \label{fig:mid-bene}
  \end{subfigure}
  \hspace{0.02\textwidth}
  \begin{subfigure}[b]{0.31\textwidth}
    \includegraphics[width=\textwidth]{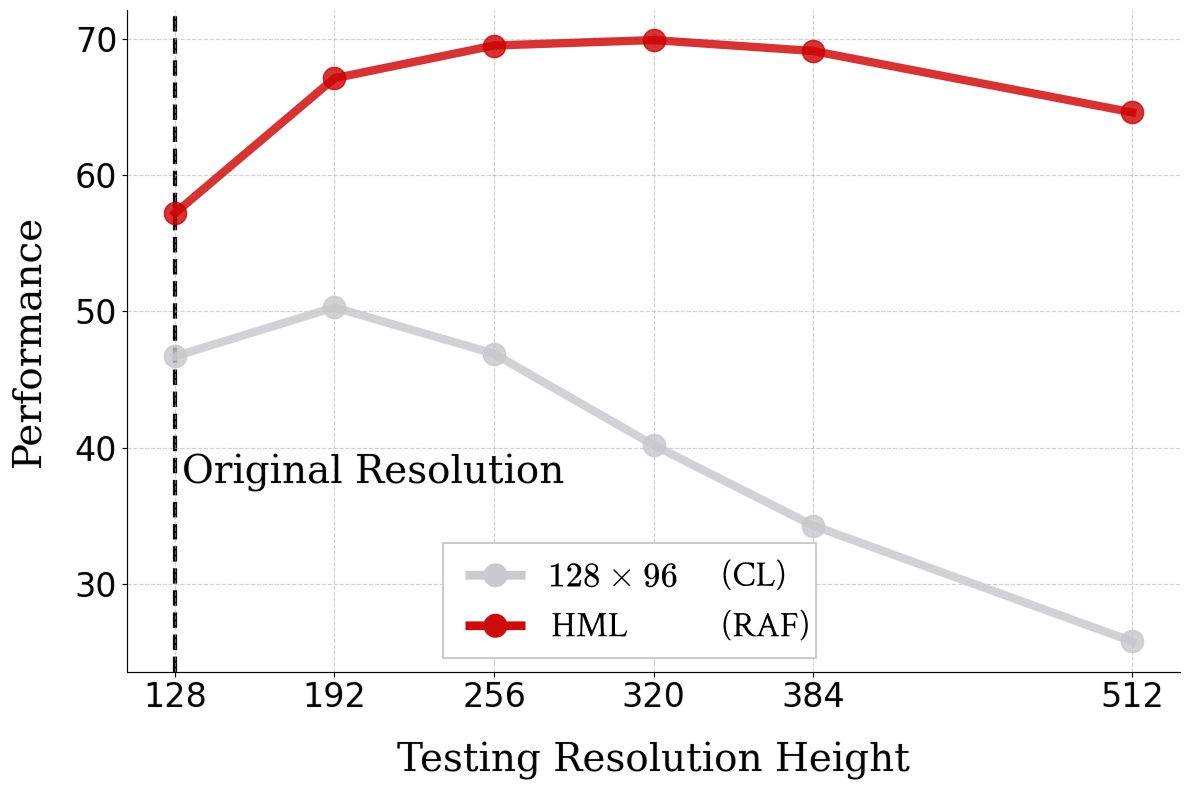}
    \caption{Low Client Benefit}
    \label{fig:low-bene}
  \end{subfigure}  
  \caption{Inference accuracy on the HPE task across various input resolutions for different training regimes. ``CL'' denotes centralized learning on a single client, ``CL+KD'' augments centralized training with our multi-resolution knowledge distillation, and both use 4,000 MPII samples. ``HML (RAF)'' denotes federated learning with three clients (high $256\times192$, mid $192\times144$, low $128\times96$), each holding 4,000 MPII samples (12,000 in total), using our RAF method. The x-axis indicates the input height, with all resolutions maintaining a 4:3 aspect ratio.}
  \label{fig:benefits-fl}
\end{figure*}

First, by comparing Base (FedAvg) with Base (FedProx), we observed similar overall performance. FedProx provides a modest $+0.4$ gain at the $192\times 144$ training resolution but underperforms FedAvg at unseen scales. While FedProx effectively addresses statistical heterogeneity in classification, it fails to mitigate resolution drift and can even introduce instability.
By contrast, both RAF (FedAvg) and RAF (FedProx) consistently outperform their baseline counterparts across all resolutions, with particularly large gains at both very low and very high scales. Without any additional data, RAF yields substantial gains over both aggregation schemes: RAF (FedAvg) improves the accuracy by $5.4\%$ at $128\times96$ and $3.9\%$ at $512\times384$, whereas RAF (FedProx) achieves $5.3\%$ and $4.9\%$ improvements at the same resolutions. Our multi-resolution distillation mechanism functions as an effective regularizer by minimizing the gap between the high- and low-resolution predictions. This prevents overfitting to any single resolution and enables RAF to generalize across a broad spectrum of input scales. These results confirm that RAF's enhanced accuracy is primarily driven by the regularizing effect of multi-resolution knowledge distillation, which specifically targets and alleviates the impact of resolution heterogeneity. Furthermore, RAF's compatibility with both FedAvg and FedProx underscores its versatility and ease of integration into existing FL frameworks.

\subsection{Analysis of Client-Side Benefits in RAF}

Although RAF serves as an effective regularizer, it introduces a modest computational overhead. A client with a high-resolution dataset has little incentive to participate in FL if the performance gains are negligible. To establish the practical benefits of combining FL with our RAF method, we conducted experiments to compare each client's performance under three settings: (1) Centralized Learning (CL) on their own data, (2) CL with our knowledge distillation (KD) method, and (3) FL with KD (RAF).

Figure~\ref{fig:high-bene}, \ref{fig:mid-bene}, and \ref{fig:low-bene} illustrate the performance benefits obtained by clients holding images at resolutions $256\times192$, $192\times144$, and $128\times96$, respectively, when participating in the proposed RAF framework.
In the high-resolution Figure~\ref{fig:high-bene} and mid-resolution Figure~\ref{fig:mid-bene}, comparing ``CL'' and ``CL+KD'' shows that applying KD in a centralized setting improves accuracy at every tested resolution and prevents severe performance degradation on unseen scales. Furthermore, when comparing ``CL+KD'' with ``RAF,'' we see similar performance gains at all resolutions. This similarity indicates that the improvement is not merely due to having more samples (as in FL), but rather due to our multi-resolution KD method, enabling each client to exploit fully the spatial information in its images. Consequently, even in the federated setting, where the sample number increases naturally, the model's accuracy increases proportionally because KD has already maximized the spatial feature utilization.
In the low-resolution Figure~\ref{fig:low-bene}, there is no ``CL+KD'' curve because no lower-resolution dataset is available for distillation. The gap between ``CL'' and ``RAF'' becomes dramatic due to the benefits of both diverse data and resolution-aware training. In particular, as inference resolution increases, the performance advantage of ``RAF'' over CL grows steadily.

These experimental results demonstrate that regardless of the client-image resolution, our KD method helps extract richer spatial information from those images. Consequently, each client gains a clear performance advantage by participating in FL using the proposed method, justifying the modest computational overhead. This broadens the FL method's applicability to heterogeneous resolutions.

\subsection{Can Clients with Only Low-Resolution Inference Still Benefit from RAF?}\label{subsec:low to high}

In the previous experiments, we demonstrated that federated learning with the proposed multi-resolution KD method achieved high accuracy across a variety of resolutions. 
However, extracting the best possible performance from data that each client actually holds differs from using a model that remains robust across all resolutions.

To benefit maximally from the performance improvements observed in earlier experiments, a client trained on low-resolution data must have access to high-resolution images at the time of inference. However, in practice, clients who cannot provide high-resolution data during training typically cannot obtain high-resolution images during inference. For example, in a scenario such as CCTV surveillance, illustrated in the upper row of Figure~\ref{fig:interpolation-inference}, we present the standard inference process for a low-resolution image. In such cases, where only low-resolution images are available for inference, the predicted keypoint locations remain somewhat blurry because of limited resolution. Consequently, when examining the low-resolution performance at $128\times96$ in Table~\ref{tab:interpolation-inference}, even a model trained using our method cannot exceed the inherent accuracy limit of $57.2$ imposed by the input resolution. 
Otherwise expressed, notwithstanding the model's power obtained through FL, its performance is ultimately constrained by the resolution of the data available for inference.

To address this limitation, we conducted additional experiments aimed at boosting performance when only low-resolution data were available. The lower row in Figure~\ref{fig:interpolation-inference} illustrates the proposed approach. 
To obtain a crisper output, we first interpolated the original low-resolution image to a higher resolution (e.g., $256\times192$) and then ran an inference on this interpolated image. The lower row of Figure~\ref{fig:interpolation-inference} shows that the high-resolution heatmap produced hereby clearly pinpoints keypoint locations compared to the directly obtained low-resolution heatmap.
Table~\ref{tab:interpolation-inference} presents the quantitative results for the various interpolation methods and resolutions.
When interpolating to $256\times192$, the model achieved \textbf{69.0}, \textbf{69.2}, and \textbf{69.7} for Bilinear, Area, and Bicubic, respectively. Interpolating to $320\times240$ yielded similarly high scores. These accuracies far exceed the standalone centralized learning performance of $46.7$ and post-FL low-resolution accuracy of $57.2$.
This demonstrates that interpolation combined with RAF can achieve significantly improved accuracy compared with an otherwise suboptimal low-resolution dataset.

\begin{table}[t]
\centering
\caption{Comparison of inference accuracy at low resolution ($128\times96$) using different interpolation methods. The model was trained using RAF across three clients holding high- ($256\times192$), mid- ($192\times144$), and low- ($128\times96$) resolution data. The ``Interpolated Res'' column indicates the post-interpolation resolution, ``$*$'' denotes the original non-interpolated resolution. The results are shown for three interpolation methods: Bilinear, Area, and Bicubic.}
\label{tab:interpolation-inference}
{\begin{tabularx}{1.0\linewidth}{>{\raggedright\arraybackslash}p{0.3\linewidth}YYYY}
\toprule
\textbf{Interpolated Res} & \textbf{Bilnear} & \textbf{Area} & \textbf{Bicubic} \\
\midrule
$128\times96$ ($*$)   & 57.2              & 57.2              & 57.2 \\
$192\times144$   & 66.7              & 66.8              & 66.9 \\
$256\times192$   & \textbf{69.0}     & 69.2              & \textbf{69.7} \\
$320\times240$   & \textbf{69.0}     & \textbf{69.4}     & 69.6 \\
$384\times288$   & 68.2              & 68.8              & 69.0 \\
$512\times384$   & 63.1              & 64.5              & 64.1 \\
\bottomrule
\end{tabularx}}
\end{table}

\begin{figure}[t]
    \centering    
    \includegraphics[width=0.93\linewidth]{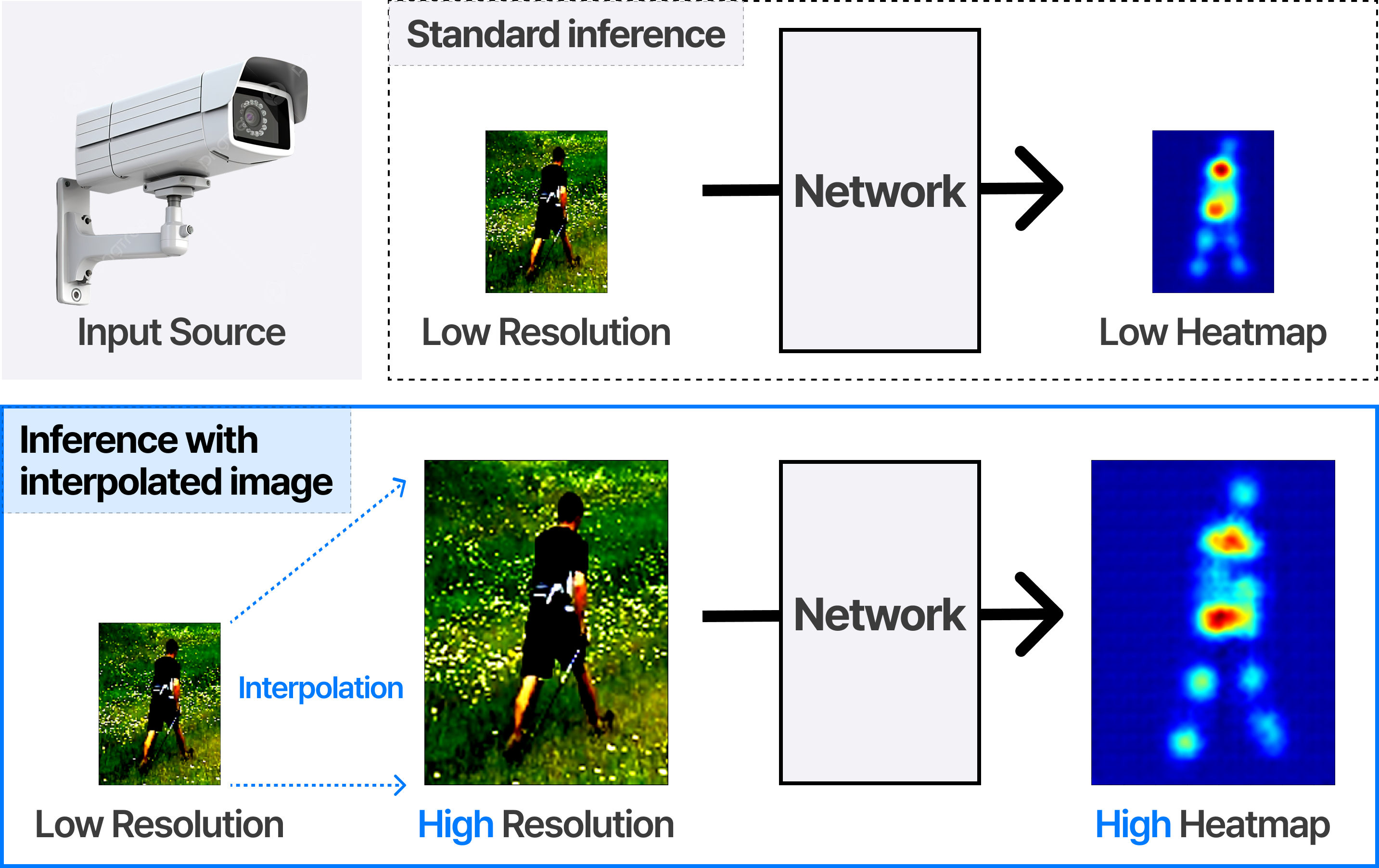}
    \caption{Inference with an interpolated image. The top row shows standard inference, while the bottom row shows inference on a low-resolution image after increasing its resolution via interpolation for improved performance.}
    \label{fig:interpolation-inference}
    \vspace{-15pt}
\end{figure}

\subsection{Scalability of RAF: Benefit to a Single High-Resolution Client with Low-Resolution Clients}

While a federated learning setup with only low-resolution clients is admittedly unrealistic, we conducted this extreme experiment to determine how far our method can ``rescue'' a high-resolution client's performance even under the worst conditions. Figure~\ref{fig:H1L11} illustrates a scenario in which one high-resolution client is joined by an increasing number of low-resolution clients (from 0 to 11). All clients were trained on $1,000$ MPII images. We report the low-($128\times96$) and high- ($256\times192$) resolution inference performances as the number of low-resolution clients increases. When there are zero low-resolution clients, the high-resolution client simply trains alone in a centralized fashion. Starting from one low-resolution client, we switched to FL, with RAF applied to each client.

\begin{figure}[t]
    \centering
    \includegraphics[width=0.93\linewidth]{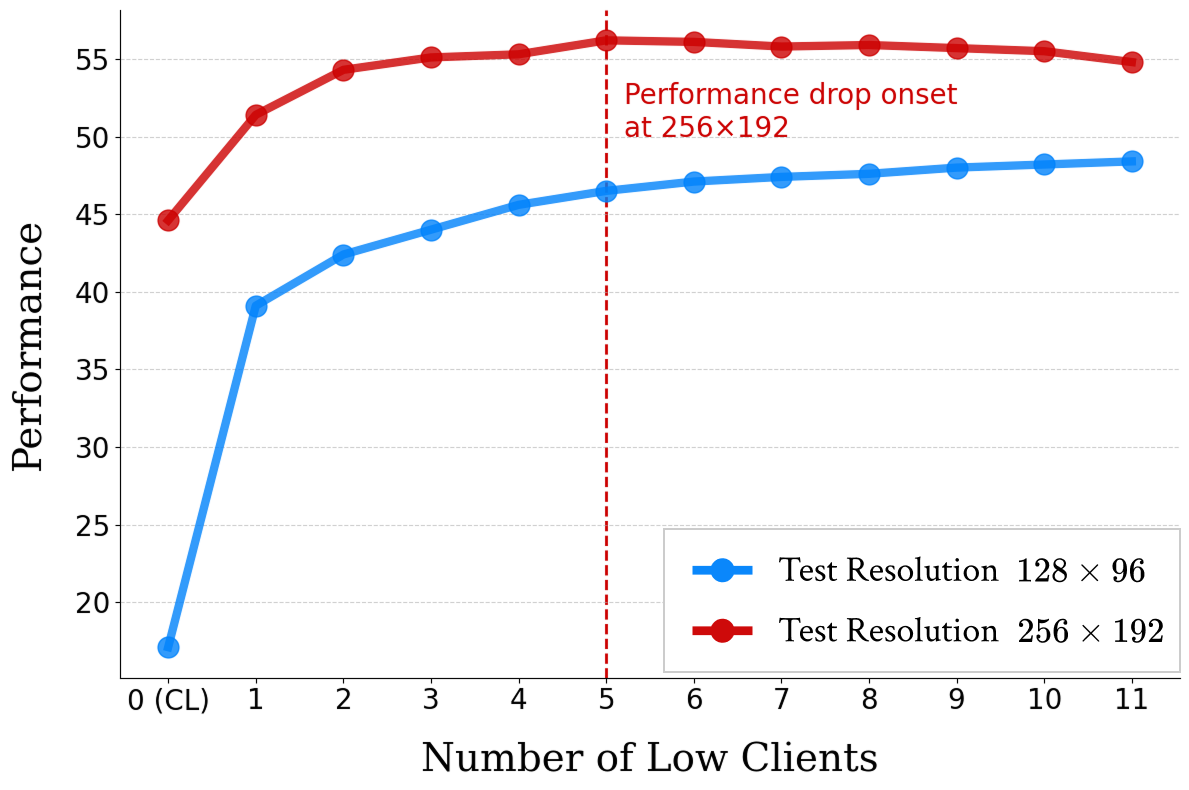}
    \caption{RAF with one high-resolution ($256\times192$) client and multiple low-resolution ($128\times96$) clients. In this setup, each client has $1,000$ samples for training.}
    \label{fig:H1L11}
    \vspace{-15pt}
\end{figure}

The low-resolution ($128\times96$) plot in Figure~\ref{fig:H1L11} shows that every FL configuration (one or more low-resolution clients) significantly outperformed the centralized baseline (zero low-resolution clients) at a low resolution. Recall from Table~\ref{tab:res_drift} that standard FL without the proposed multi-resolution KD actually degrades the performance when aggregating dissimilar resolutions, such as high and low, despite the greater spatial detail in higher-resolution inputs. In contrast, RAF shows that even when only low-resolution images are available during inference, the model learns to leverage every bit of spatial information, demonstrating the effectiveness of our KD scheme. Moreover, as more low-resolution clients were added, the low-resolution performance steadily improved. This implies that a purely low-resolution FL setup still yields benefits for each low-resolution participant, guaranteeing that they gain by joining, even though a high-resolution client is alone among many low-resolution peers.

In the high-resolution ($256\times192$) plot, we observed that all FL variants (with one or more low-resolution clients) outperformed the centralized high-resolution baseline. As we increased the number of low-resolution clients to five, the high-resolution performance continued to increase, peaking when five low-resolution participants were present. Beyond the five low-resolution clients, the performance began to dip slightly. We interpret this as follows: for up to five low-resolution clients, our KD mechanism's resolution-generalization effect applied to a single high-resolution client dominates and continues to improve the performance. However, when there are more than five low-resolution clients, the KD module cannot fully absorb the vast increase in low-resolution data, leading to a slight overfitting of low-resolution features, and consequently, a minor drop in high-resolution accuracy.
In summary, our experiments demonstrate that a single high-resolution client augmented with our KD method can be generalized across up to five low-resolution peers without overfitting. Even when faced with a much larger number of low-resolution clients, the high-resolution client still achieved a much higher performance than it would when training alone in a centralized setting. Therefore, regardless of the adverse effects in the FL scenario, a high-resolution client can guarantee a performance gain by adopting the proposed method.

\begin{figure}[htbp]
    \centering
    \includegraphics[width=0.95\linewidth]{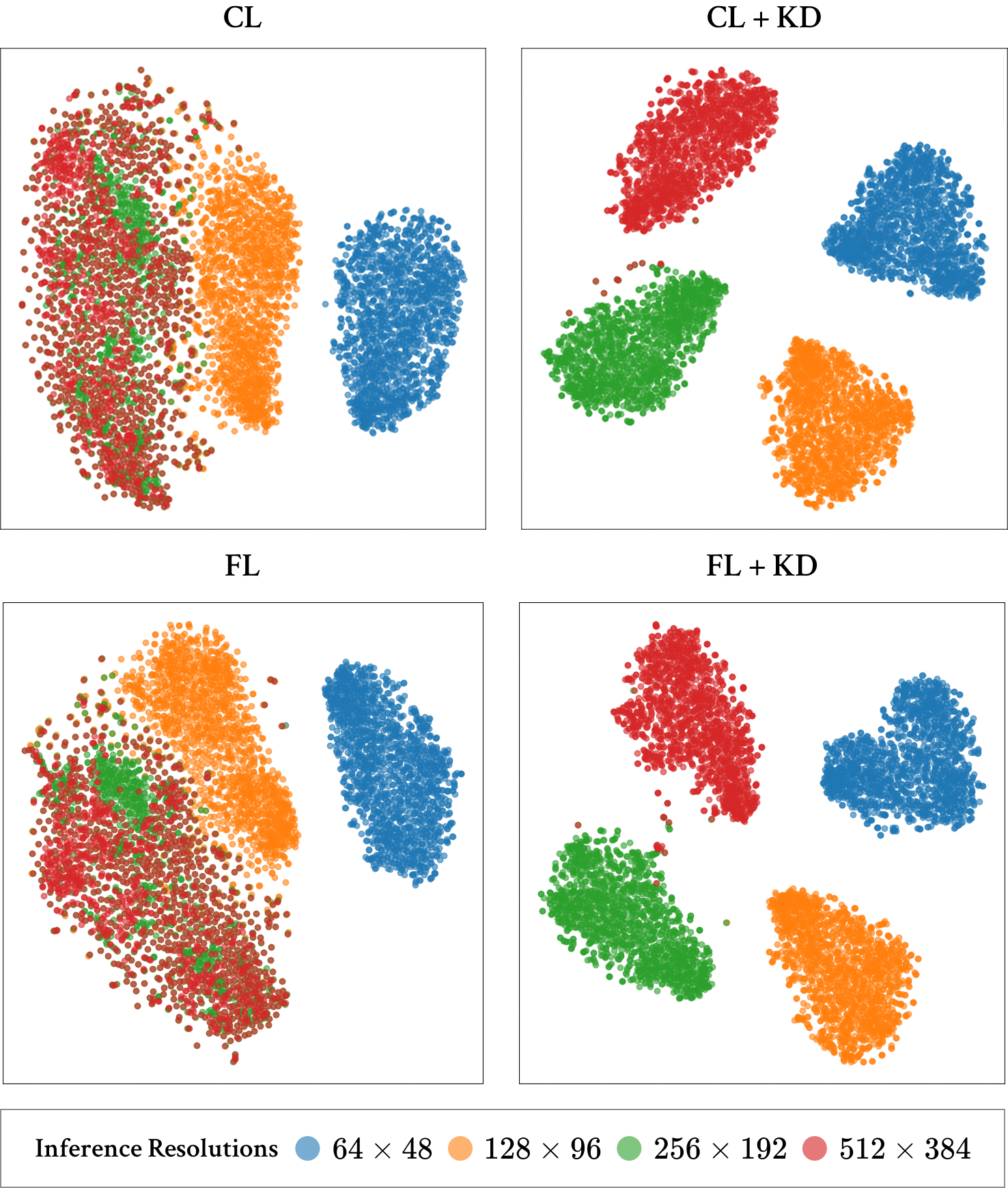}
    \caption{t-SNE visualization of feature embeddings (from the last ViT block) colored by inference resolution. ``CL'' denotes centralized learning on a single high-resolution client ($256\times192$); ``FL'' denotes federated learning with three clients holding high ($256\times192$), mid ($192\times144$), and low ($128\times96$) resolution data. ``CL+KD'' and ``FL+KD'' apply our multi-resolution knowledge distillation to the CL and FL setups, respectively.}
    \label{fig:t-sne-comp}
    \vspace{-10pt}
\end{figure}

\subsection{t-SNE Analysis: Resolution Robustness via RAF}

\paragraph{Anlaysis}
Figure~\ref{fig:t-sne-comp} shows an intuitive t-SNE analysis that demonstrates the resolution robustness of the models trained using RAF. Without multi-resolution KD, the model struggles to distinguish between the inputs at $256\times192$ and $512\times384$. Although this can separate some of the other scales, it mistakes the largest unseen resolution, $(512\times 384)$. In contrast, models trained with the proposed multi-resolution KD formed four clearly distinct clusters corresponding to each resolution, both in the centralized and FL settings. This striking separation indicates that RAF both enhances resolution robustness and internally teaches the network to recognize and adapt to different input scales.

\paragraph{Discussion}
In Section~\ref{subsec:cls-hrr}, we highlight that classification models discard spatial information from the input image, whereas high-resolution regression tasks must preserve this information, and are therefore sensitive to the input resolution. Although our primary focus was human pose estimation, we further illustrated this fundamental difference using a semantic segmentation task. Figure~\ref{fig:tsne-cls-hrr} presents t-SNE visualizations of feature embeddings from models trained on classification and segmentation. The classification model collapsed the features from all resolutions into a single cluster, clearly confirming that it ignored the spatial structure of the input. In contrast, the segmentation model's embeddings formed distinct clusters by resolution, despite some overlap between $256\times192$ and $512\times384$, demonstrating that high-resolution representation models remain sensitive to changes in the input scale and require richer feature representations.
This pattern holds for HPE and for segmentation, indicating that the resolution sensitivity we identified is a general property of high-resolution representation models. Consequently, we believe that our RAF framework is applicable beyond human pose estimation and can benefit a wide range of high-resolution representation tasks.

\begin{figure}[t]
    \centering
    \includegraphics[width=0.95\linewidth]{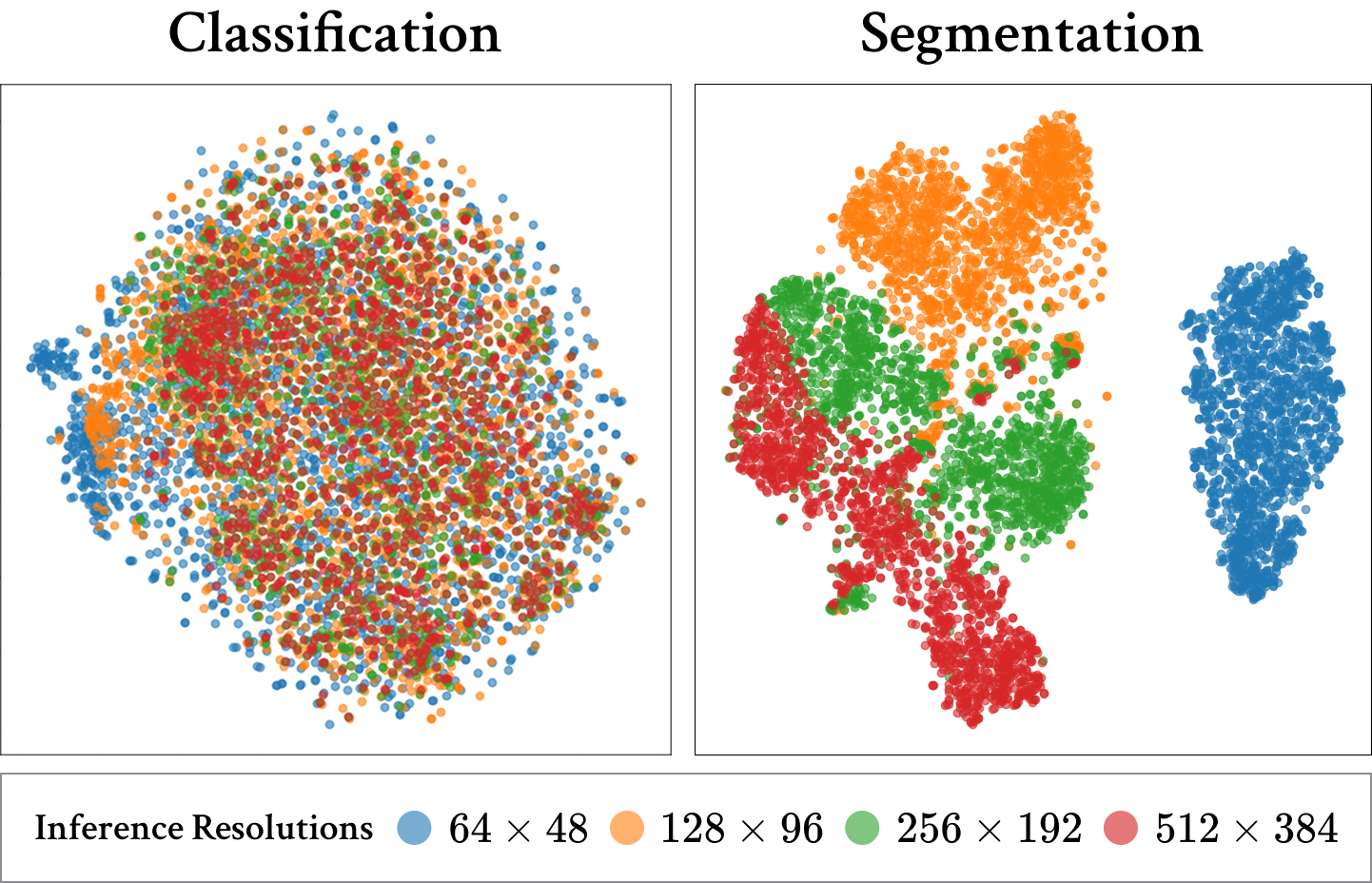}
    \caption{t-SNE visualizations of feature embeddings, colored by inference resolution, obtained from models trained on classification versus high-resolution regression tasks. Inputs at four scales are passed through a ResNet-50~\cite{resnet} model trained for classification (left) and an FCN~\cite{fcn} trained for semantic segmentation (right). Both models were taken from the standard implementations provided by Torchvision.}
    \vspace{-10pt}
    \label{fig:tsne-cls-hrr}
\end{figure}

\section{Conclusion and Future Work}

This paper identified a significant performance degradation termed \emph{ resolution drift } that occurs when clients with different input resolutions collaborate on high-resolution regression tasks in an FL setting. To address this issue, we proposed and investigated RAF, a framework that augments local training using multi-resolution knowledge distillation. By acting as a resolution-aware regularizer, RAF prevents a model from overfitting to any single scale and substantially improves resolution robustness. Our extensive experiments demonstrated that RAF effectively mitigated resolution drift, and we complemented these empirical findings with a theoretical convergence analysis. Because RAF functions as a local training augmentation process, it is orthogonal to existing FL aggregation algorithms and can seamlessly integrate into a wide range of FL pipelines.

While our evaluation focused on human pose estimation, we believe that the RAF approach potentially extends to many other non-classification tasks. In future work, we will apply RAF to additional high-resolution representation problems, such as semantic segmentation, depth estimation, and super-resolution.

\bibliographystyle{IEEEtran}

\newpage

\end{document}